\documentclass[twoside]{article}

\usepackage[accepted]{aistats2026}
\usepackage[round]{natbib}

\usepackage{amsthm}
\usepackage{algorithm}
\usepackage{algorithmic}
\usepackage{blindtext}
\usepackage{titlesec}
\usepackage{microtype}
\usepackage{graphicx}
\usepackage{subfigure}
\usepackage{booktabs} 
\usepackage[cjk]{kotex}
\usepackage{amsfonts, amsmath, amssymb, amscd, exscale, epic, eepic, epsfig, dsfont}
\usepackage{mathtools}
\usepackage{fancyhdr,hyperref}
\usepackage[capitalize, nameinlink]{cleveref}
\usepackage{bigints}
\usepackage{float}
\usepackage{relsize}
\usepackage{color}
\usepackage{tikz-cd}
\usepackage{makeidx}
\usepackage{caption}
\usepackage{bm}
\usepackage{listings,comment}
\usepackage{mathrsfs}
\usepackage{enumerate, makecell, multirow, indentfirst}
\usepackage[utf8]{inputenc} 
\usepackage[T1]{fontenc}    
\usepackage[english]{babel}

\newcommand{\bSigma}{\mathbf{\Sigma}}

\newcommand{\Ecr}{\mathscr{E}}

\newcommand{\bb}{\mathbf{b}}

\newcommand{\pb}{\mathbf{p}}
\newcommand{\qb}{\mathbf{q}}

\newcommand{\Ab}{\mathbf{A}}
\newcommand{\Bb}{\mathbf{B}}

\newcommand{\Ib}{\mathbf{I}}

\newcommand{\Vb}{\mathbf{V}}

\newcommand{\cA}{\mathcal{A}}

\newcommand{\cH}{\mathcal{H}}

\newcommand{\cN}{\mathcal{N}}
\newcommand{\cO}{\mathcal{O}}

\newcommand{\cX}{\mathcal{X}}
\newcommand{\cY}{\mathcal{Y}}
\newcommand{\cZ}{\mathcal{Z}}

\newcommand{\EE}{\mathbb{E}} 
\newcommand{\PP}{\mathbb{P}} 
\newcommand{\RR}{\mathbb{R}} 
\ifx\BlackBox\undefined
\newcommand{\BlackBox}{\rule{1.5ex}{1.5ex}}  
\fi

\ifx\QED\undefined
\def\QED{~\rule[-1pt]{5pt}{5pt}\par\medskip}
\fi

\usepackage{etoc}
\etocdepthtag.toc{mtchapter}
\etocsettagdepth{mtchapter}{subsection}
\etocsettagdepth{mtappendix}{none}

\usepackage{siunitx}

\theoremstyle{plain} 
\newtheorem{lemma}{\textbf{Lemma}} 

\newtheorem{theorem}{\textbf{Theorem}}
\newtheorem{corollary}{\textbf{Corollary}}
\newtheorem{assumption}{\textbf{Assumption}}

\newtheorem{definition}{\textbf{Definition}}
 
\newtheorem{problem}{\textbf{Problem}}
\newtheorem{proposition}{\textbf{Proposition}}

\theoremstyle{definition}

\newcommand{\cov}{\operatorname{cov}}
\newcommand{\xor}{\operatorname{xor}}
\begin{document}

%

%

\twocolumn[

\aistatstitle{Nearly Optimal Best Arm Identification for Semiparametric Bandits}

\aistatsauthor{ Seok-Jin Kim }

\aistatsaddress{ Columbia University } ]

\begin{abstract}
We study fixed-confidence Best Arm Identification (BAI) in semiparametric bandits, where rewards are linear in arm features plus an unknown additive baseline shift.
Unlike linear-bandit BAI, this setting requires orthogonalized regression, and its instance-optimal sample complexity has remained open.
For the transductive setting, we establish an attainable instance-dependent lower bound characterized by the corresponding linear-bandit complexity on shifted features.
We then propose a computationally efficient phase-elimination algorithm based on a new $\mathcal{X}\mathcal{Y}$-design for orthogonalized regression.
Our analysis yields a nearly optimal high-probability sample-complexity upper bound, up to log factors and an additive $d^2$ term, and experiments on synthetic instances and the Jester dataset show clear gains over prior baselines.
\end{abstract}


\section{Introduction}

In the multi-armed bandit (MAB) framework, a learner sequentially selects actions and receives rewards from unknown distributions \citep{LS19bandit-book}. Beyond regret minimization, a major line of work studies \textit{Best Arm Identification} (BAI) \citep{soare2014best,kaufmann2016complexity,jamieson2014best}, where the objective is to identify the optimal action. 
The two standard formulations are the fixed-confidence setting \citep{fiez2019sequential,soare2014best} and the fixed-budget setting \citep{degenne2020gamification,jedra2020optimal}. We focus on the \textit{fixed-confidence} setting: given a risk parameter $\delta$, the learner must return the best arm with probability at least $1-\delta$. Performance is measured by the resulting \textit{sample complexity}.

Classical MAB models treat each arm as an unrelated option \citep{auer2002finite,langford2007epoch}. In many applications, however, arms are associated with feature vectors and rewards are structured through those features. When the mean reward is linear in the feature vector, one obtains the \textit{linear bandit} model, which has been studied extensively \citep{goldenshluger2013linear,abbasi2011improved,li2019nearly,rusmevichientong2010linearly}.

Purely linear models are often too restrictive for practice. In mobile health, for example, a user's baseline health state may vary in complex and non-stationary ways that are unrelated to the treatment effect \citep{greenewald2017action}. In such treatment-regime problems, each arm corresponds to a treatment, and ``do nothing'' often acts as a reference arm. The reward of that reference arm is a nuisance component that may be highly complex, while the \textit{treatment effect} remains comparatively simple. Motivated by this setting, \citet{greenewald2017action} studied a model in which the baseline reward is arbitrary but the treatment effect, namely the difference between the reward of treatment $x_i$ and the reference $x_1$, follows a linear parametric model. Related semistructured models, where the nuisance component is complex but the effect of interest is simple, also appear in treatment-effect estimation \citep{kennedy2023towards,kennedy2024minimax} and sequential decision-making \citep{wen2025joint}.

This idea is captured by the \textbf{semiparametric reward model} of \citet{krishnamurthy2018semiparametric}. The model has been studied in contextual bandits \citep{greenewald2017action,krishnamurthy2018semiparametric,kim2019contextual} and, more recently, in the fixed-feature setting by \citet{pmlr-v291-kim25a}. Formally, the reward $r_t$ at time $t$ for action $a_t$ with feature vector $x_{a_t} \in \mathbb{R}^d$ is
\[
r_t = x_{a_t}^\top \theta^\star + \nu_t + \eta_t,
\]
where $\theta^\star \in \RR^d$ is an unknown parameter vector, $\nu_t$ is an arbitrary baseline shift chosen before action selection, and $\eta_t$ is random noise. This extends the linear bandit model (the special case $\nu_t = 0$) by allowing complex baseline dynamics such as those arising in recommender systems \citep{kim2019contextual} and clinical trials. Because BAI is typically carried out over a fixed set of candidate actions, robust experimental design under such baseline shifts is essential. \citet{pmlr-v291-kim25a} referred to this fixed-feature setting as \textbf{semiparametric bandits} and established a $\tilde{\mathcal{O}}(\sqrt{dT})$ regret bound together with the first BAI results for this setting.

\paragraph{Motivating Examples.} The baseline-shift model naturally captures various real-world scenarios:
\begin{itemize}
    \setlength\itemsep{0em}
    \item \emph{Clinical Trials:} Trials often involve fixed treatments administered to varying subjects. Here, $\nu_t$ represents the complex, non-stationary baseline response of the $t$-th subject (e.g., the response under ``do nothing''), while the treatment effect remains stable across the population \citep{greenewald2017action}.
    \item \emph{Ad Selection:} In landing page optimization without personal data, $\nu_t$ can represent the baseline clicking propensity of a user arriving at time $t$, reflecting external trends or user heterogeneity \citep{pmlr-v291-kim25a}.
    \item \emph{Human Ratings (Jester):} In the Jester joke rating dataset \citep{goldberg2001eigentaste}, users exhibit distinct tendencies to rate all items high or low, regardless of the specific joke. The term $\nu_t$ captures this evaluator-specific baseline, while the intrinsic joke quality remains parametric. We return to this example in Section~\ref{section; real_world_jester}.
\end{itemize}

\paragraph{Transductive Fixed-Confidence BAI.}
We study BAI in the more general \textit{transductive} setting \citep{fiez2019sequential}, where data collection and final recommendation need not share the same action set. The learner can sample from a \emph{source} feature set $\mathcal{X} = \{x_1, \dots, x_K\} \subset \mathbb{R}^d$, but must identify the best arm in a possibly different \emph{target} feature set $\mathcal{Z} = \{z_1, \dots, z_H\} \subset \mathbb{R}^d$. After collecting data from $\mathcal{X}$, the learner recommends $\hat{z} \in \mathcal{Z}$, where the true target optimum $z^\star = \arg\max_{z \in \mathcal{Z}} z^\top \theta^\star$ is assumed unique.
This formulation models settings in which experimentation is constrained to a design set, while the downstream decision is taken over a broader or different evaluation set. For example, in drug development many compounds and dosages may be testable in the lab ($\cX$), but only a subset is approved for deployment ($\cZ$) \citep{fiez2019sequential}. The usual non-transductive setting is recovered by taking $\mathcal{Z}=\mathcal{X}$, so all of our results immediately specialize to the standard case. 

\subsection{Research Questions}
\label{subsection; research questions}

We study fixed-confidence transductive BAI for semiparametric bandits. While \citet{pmlr-v291-kim25a} initiated the study of experimental design in this setting, it remains unclear whether their sample-complexity guarantees are optimal. More fundamentally, instance-dependent lower bounds for semiparametric BAI were not previously known. This leads to two central questions:

\begin{itemize}
    \item[\textbf{Q1.}] What is the attainable lower bound on the sample complexity for fixed-confidence BAI in semiparametric bandits?
    \item[\textbf{Q2.}] Can we design an algorithm that matches this lower bound?
\end{itemize}

In linear bandits, instance-optimal BAI relies on $\mathcal{XY}$-designs rather than G-optimal designs \citep{fiez2019sequential}. Semiparametric bandits, however, require \textbf{orthogonalized regression} to obtain consistent estimates; see Section~\ref{subsection; orthogonalized regression and policy design}. Although G-optimal design for orthogonalized regression is known \citep{pmlr-v291-kim25a}, the corresponding $\mathcal{XY}$-design problem has not been developed. Answering Q1 and Q2 therefore requires us to solve two challenges: (C1) identify an attainable instance-dependent lower bound, and (C2) construct an effective $\mathcal{XY}$-design for orthogonalized regression.

\subsection{Result Overview and Contribution}

Our main contributions are as follows:

\begin{itemize}
    \item \textbf{Lower bound:} We prove the first attainable instance-dependent lower bound for BAI in semiparametric bandits.
    \item \textbf{Algorithm:} We propose an efficient phase-elimination algorithm for transductive semiparametric BAI.
    \item \textbf{Sample-complexity guarantee:} We prove a high-probability upper bound controlled by the shifted linear-bandit benchmark up to logarithmic factors and an additive \(d^2\) term. The same benchmark also appears in our lower bound, and the guarantee immediately extends to the standard non-transductive case $\mathcal{Z}=\mathcal{X}$.
    \item \textbf{$\mathcal{XY}$-design for orthogonalized regression:} We introduce a new design construction that controls contrast variances for orthogonalized regression, which is the key technical ingredient behind our upper bound.
\end{itemize}

\paragraph{Comparison with Prior Work.}
To compare with prior semiparametric BAI results, we specialize here to the standard non-transductive setting \(\cZ=\cX\) and identify the common arms by \(z_i=x_i\) for \(i\in[K]\). Define \(q_i := z_i^\top \theta^\star\), let \(q^\star := \max_{i \in [K]} q_i\), and let \(q_{(1)} \ge q_{(2)} \ge \cdots\) denote the sorted expected rewards. Define the gaps \(\Delta_j := q^\star - q_{(j)}\), so \(0=\Delta_1 \le \Delta_2 \le \cdots \le \Delta_K\) and \(\Delta_2\) is the smallest nonzero gap.
In this non-transductive setting, prior semiparametric BAI algorithms such as SBE and G-Opt \citep{pmlr-v291-kim25a} guarantee sample complexity \(\tilde{\mathcal{O}}(d/\Delta_2^2)\). By contrast, our algorithm scales as \(\tilde{\mathcal{O}}(\tau_{\text{lin}}^*(\mathcal{X} - x_1))\), where \(\tau_{\text{lin}}^*(\mathcal{X} - x_1)\) is the optimal fixed-confidence sample complexity of the corresponding linear bandit instance with shifted feature set \(\cX-x_1\). Since \(\tau_{\text{lin}}^*(\mathcal{X} - x_1) \lesssim d/\Delta_2^2\) always holds and can be much smaller on favorable instances, our guarantee is never worse and can be substantially tighter. Appendix~\ref{section; apdx related work} gives a broader comparison with related semiparametric, linear-bandit, and corruption-robust literatures.


\subsection{Notations}
\label{subsection; notations}

We write $[n] := \{1, 2, \ldots, n\}$ for a positive integer $n$ and denote the $n$-dimensional simplex by $\Delta^{(n)}$. For a vector $x \in \mathbb{R}^d$, $\|x\|_p$ denotes the $\ell_p$ norm. For any positive semidefinite matrix $\mathbf{A}$, we define the weighted norm $\|x\|_\mathbf{A} = \sqrt{x^\top \mathbf{A} x}$. We use $\mathcal{O}(\cdot)$ or $\lesssim$ to suppress absolute constants, and $\tilde{\mathcal{O}}(\cdot)$ to further suppress logarithmic factors in the natural problem parameters when this causes no confusion.
The notation $a \asymp b$ means $a \lesssim b$ and $b \lesssim a$. \textit{Constants $c, C, c_1, \dots$ may change from line to line.} For sets $A, B \subset \mathbb{R}^n$ and a vector $a \in \mathbb{R}^n$, we define $A - a := \{a' - a \mid a' \in A\}$ and $A-B := \{a-b \mid a \in A, b \in B\}$.

Following \citet{pmlr-v291-kim25a}, we generalize the matrix-inverse-weighted norm $\| x\|_{\mathbf{A}^{-1}}$ to positive semidefinite matrices $\mathbf{A}$ (regardless of invertibility) as:
\[
\|x\|_{\mathbf{A}^{-1}} := \lim_{\lambda \to 0^+} \|x\|_{(\mathbf{A} + \lambda \mathbf{I}_d)^{-1}}.
\]
This definition coincides with the standard definition when $\mathbf{A}$ is full rank and remains well-defined and finite if $x$ lies in the column space of $\mathbf{A}$.
If $x$ does not lie in the column space of $\mathbf{A}$, then $\|x\|_{\mathbf{A}^{-1}}=+\infty$ under this convention. 

\section{Setup and Preliminaries}
\label{section; background}

\subsection{Fixed-confidence BAI}
\label{subsection; fixed confidence BAI}

For a confidence level \(\delta \in (0,1)\), a fixed-confidence algorithm sequentially selects source arms, stops at a random time \(\bm\tau(\delta)\), and outputs an estimate \(\hat z \in \cZ\) of the best target arm \(z^\star\). We require the algorithm to be \(\delta\)-correct:
\begin{align*}
\PP[\hat z \ne z^\star ] \le \delta.
\end{align*}
When this condition holds, we refer to the stopping time \(\bm \tau(\delta)\) as the \textit{sample complexity}.

\paragraph{Lower bound for transductive linear bandits.}
Consider a linear bandit problem with a \emph{source} feature set \(\cX\) and a \emph{target} feature set \(\cZ\). Let \(z^\star\) be the unique maximizer of \(z^\top \theta^\star\) over \(z\in\cZ\).
The optimal sample complexity for transductive BAI admits the following instance-dependent lower bound \citep{fiez2019sequential}:
\begin{align}\label{equation; lower bound transductive linear bandit}
    \bm\tau_{\operatorname{lin}}^\star(\cZ : \cX) :=
 \min_{\pb \in \Delta^{(K)}} \max_{z \in \cZ \setminus \{z^\star\} } \frac{\|z^\star-z\|_{(\sum_{i=1}^K p_i x_i x_i^\top)^{-1}}^2 }{|(z^\star-z)^\top \theta^\star|^2}.
\end{align}
The numerator captures the statistical resolution required to estimate the contrast \((z^\star-z)^\top \theta^\star\) under design \(\pb\). Thus \eqref{equation; lower bound transductive linear bandit} has the familiar ``resolution versus gap'' form: to certify that \(z^\star\) beats a competitor \(z\), the learner must estimate the corresponding contrast to an accuracy commensurate with the gap size. The design \(\pb\) controls the contrast variance through the inverse information matrix \((\sum_i p_i x_i x_i^\top)^{-1}\).
Under the generalized inverse-norm convention above, this benchmark also captures identifiability. In particular, if some relevant contrast \(z^\star-z\) lies outside \(\operatorname{span}(\cX)\), then the numerator is \(+\infty\) for every design \(\pb\), and hence \(\bm\tau_{\operatorname{lin}}^\star(\cZ : \cX)=+\infty\).

As a special case, setting \(\cZ=\cX\) recovers the standard (non-transductive) linear-bandit BAI lower bound:
\begin{align}\label{equation; lower bound linear bandit}
   &\bm\tau_{\operatorname{lin}}^\star(\cX) := \bm\tau_{\operatorname{lin}}^\star(\cX : \cX) \\
 &=\min_{\pb \in \Delta^{(K)}} \max_{x \in \cX \setminus \{x^\star\} } \frac{\|x^\star-x\|_{(\sum_{i=1}^K p_i x_i x_i^\top)^{-1}}^2 }{|(x^\star-x)^\top \theta^\star|^2}.
\end{align}
Later, for the shifted non-transductive benchmark, we use the shorthand
\[
\bm\tau_{\operatorname{lin}}^\star(\cX-x_1) := \bm\tau_{\operatorname{lin}}^\star(\cX : \cX-x_1).
\]

\paragraph{Attainability via \(\cX\cY\)-design.}
To achieve nearly optimal sample complexity in linear-bandit BAI, one uses \(\cX\cY\)-design rather than classical G-optimal design \citep{fiez2019sequential}. For an active target set \(\cA \subseteq \RR^d\) (typically \(\cA \subset \cZ\) during elimination), the \(\cX\cY\)-design problem is
\begin{align*}
     \min_{\pb \in \Delta^{(K)}} \max_{u,u' \in \cA} (u-u')^\top \Big(\sum_{i=1}^K p_i x_i x_i^\top\Big)^{-1} (u-u').
\end{align*}
Algorithms based on this design achieve nearly optimal complexity \citep{fiez2019sequential}.

The reason is that G-optimal design minimizes \(\max_{x\in\cX} x^\top(\sum_i p_i x_i x_i^\top)^{-1}x\), which is appropriate for \emph{prediction} over \(\cX\). In BAI, the relevant quantities are instead the \emph{pairwise contrasts} within the surviving set \(\cA\subseteq\cZ\). \(\cX\cY\)-designs are tailored precisely to control those contrast variances. This distinction becomes even more important in semiparametric bandits, where estimation is based on orthogonalized, centered features.


\subsection{Our Setup: Semiparametric Bandits}
\label{subsection; semiparametric bandits with fixed contexts}

The semiparametric bandit model with fixed source features was formalized by \citet{pmlr-v291-kim25a}. Let \(\cX = \{x_1,\dots, x_K\} \subset \mathbb{R}^d\) denote the \emph{source} feature set. At each time \(t = 1, 2, \ldots\), the learner selects a source arm \(a_t \in [K]\) and observes
\[
r_t = x_{a_t}^\top \theta^\star + \nu_t + \eta_t,
\]
where \(\nu_t \in \mathbb{R}\) is an arbitrary bounded shift and \(\eta_t\) is independent sub-Gaussian noise with proxy \(1\). Let \(\mathcal{H}_{t-1}\) be the history sigma-algebra generated by \(\{a_1, r_1, \dots, a_{t-1}, r_{t-1}\}\). We assume that \(\nu_t\) is \(\mathcal{H}_{t-1}\)-measurable, so the baseline may be fully adaptive to the past.

In the transductive BAI problem, we are also provided with a target feature set \(\cZ=\{z_1,\dots,z_H\}\subset\RR^d\).
The objective is to identify the unique best arm
\(
z^\star = \arg\max_{z\in\cZ} z^\top \theta^\star
\)
with probability at least \(1-\delta\).

Following prior work \citep{pmlr-v291-kim25a,krishnamurthy2018semiparametric,kim2019contextual,choi2023semi}, we impose the following boundedness assumption:
\begin{assumption}[Boundedness]
\label{assumption; boundedness}
For all \(t \geq 1\), \(i \in [K]\), and \(h \in [H]\), we assume \(|\nu_t| \leq 1\), \(\|x_i\|_2 \leq 1\), and \(\|z_h\|_2 \leq 1\).
We also assume \(\|\theta^\star\|_2 \leq 1\).
\end{assumption}

The term \(\nu_t\) represents a complex baseline or drift that is \emph{not} modeled parametrically. Because it may adapt to past observations, standard least-squares regression on \((x_{a_t},r_t)\) is not appropriate. We therefore use \textbf{orthogonalized regression}, which mitigates the effect of \(\nu_t\) by regressing on centered features \(x_{a_t}-\bar x_{\pb}\).

\subsection{Orthogonalized Regression and Policy Design}
\label{subsection; orthogonalized regression and policy design}

Orthogonalized regression was developed precisely for the semiparametric reward model \citep{krishnamurthy2018semiparametric,kim2019contextual,pmlr-v291-kim25a}. Given data \(\{ x_{a_s}, r_s\}_{s=1}^t\) generated under a fixed policy \(\pb =(p_1, \dots, p_K)\in \Delta^{(K)}\), define the policy mean \(\bar x_{\pb} := \sum_{i=1}^K  p_i x_i \). The estimator is the ridge-regression fit obtained from centered covariates \(x_{a_s} - \bar x_{\pb}\):
\begin{align*}
    &(\beta \Ib_d +\sum_{s=1}^t (x_{a_s}-\bar x_\pb)(x_{a_s}-\bar x_\pb)^\top)^{-1} \sum_{s=1}^t (x_{a_s}-\bar x_\pb) r_s.
\end{align*}
To see why the baseline does not bias the estimator, fix a policy \(\pb\) and write
\[
\tilde{x}_{a_s} := x_{a_s}-\bar x_\pb,
\qquad
\widehat{\Vb}_t := \sum_{s=1}^t \tilde{x}_{a_s}\tilde{x}_{a_s}^\top.
\]
Since \(x_{a_s} = \tilde{x}_{a_s} + \bar x_\pb\), the estimation error admits the decomposition
\begin{align*}
\hat{\theta}_t - \theta^\star
&= (\widehat{\Vb}_t + \beta \Ib_d)^{-1}
\sum_{s=1}^t \tilde{x}_{a_s}\bigl(\bar x_\pb^\top \theta^\star + \nu_s + \eta_s\bigr) \\
&\quad - \beta(\widehat{\Vb}_t + \beta \Ib_d)^{-1}\theta^\star.
\end{align*}
Under the fixed policy \(\pb\), we have \(\EE[\tilde{x}_{a_s}\mid \cH_{s-1}] = 0\). Therefore
\[
\EE\!\left[\tilde{x}_{a_s}\bigl(\bar x_\pb^\top \theta^\star + \nu_s\bigr)\mid \cH_{s-1}\right] = 0,
\]
because \(\bar x_\pb^\top \theta^\star + \nu_s\) is \(\cH_{s-1}\)-measurable. In particular, the nuisance shift \(\nu_s\) cancels from the estimating equation in conditional expectation. The remaining term \(\sum_{s=1}^t \tilde{x}_{a_s}\eta_s\) is the usual mean-zero noise term, while the final term is the regularization bias. This is the basic reason orthogonalized regression remains consistent even when the baseline shift is fully adaptive to the past.

\paragraph{Concentration Bounds.}
\citet{pmlr-v291-kim25a} established sharp concentration bounds for this estimator. Define the covariance of a policy \(\pb\) by
\begin{align*}
&\bSigma_{\operatorname{cov},\pb} := \sum_{i=1}^K p_i(x_i - \bar x_\pb)(x_i - \bar x_\pb)^\top, \\
&\text{where } \bar{x}_\pb := \sum_{i=1}^K p_i x_i.
\end{align*}
Suppose \(t\) samples are collected under policy \(\pb\). If the regularizer satisfies \(\beta \asymp \log(t/\delta)\), and for a vector \(y\) of interest there exist constants \(L, M > 0\) such that
\begin{align*}
    &y^\top \bSigma_{\operatorname{cov},\pb}^{-1} y \le L,
    \qquad
    \max_{i \in [K]} \| x_i- \bar x_\pb \|^2_{\bSigma_{\operatorname{cov},\pb}^{-1}} \le M,
\end{align*}
then
\begin{align}\label{equation; concentration orthogonalized regression}
|y^\top(\hat\theta_t - \theta^\star) | \leq c_1\frac{\sqrt{L \log(t/\delta)}}{\sqrt{t}} + c_2\frac{M\sqrt{L} \log(d/\delta)}{t},
\end{align}
for absolute constants \(c_1, c_2 >1\).
Consequently, to resolve a gap \(\Delta\), it suffices to take
\begin{align}\label{equation; sample complexity orthogonalized regression}
n \gtrsim R_1\frac{L}{\Delta^2} \log\left(\frac{L}{\Delta \delta}\right) + R_2 \frac{ M\sqrt{L}}{\Delta} \log\left(\frac{d}{\delta}\right),
\end{align}
to guarantee \(|y^\top (\hat \theta_t - \theta^\star)| \le \Delta\) for some absolute constants \(R_1, R_2>0\). For a finite family of contrasts, the same bound holds uniformly after a union bound. The proof-level values of these constants are in fact moderate (well below \(10^2\)), but as in many modern ML and bandit works, we treat their practical counterparts as tunable hyperparameters.

The dominant term in \eqref{equation; sample complexity orthogonalized regression} scales with \(L/\Delta^2\), where \(L\) is the inverse-covariance norm of the relevant contrast under policy \(\pb\). The second term, involving the source-only stability quantity \(M\), is lower-order in \(t\).

\paragraph{Approximate G-optimal design.}
Since the leading term in \eqref{equation; sample complexity orthogonalized regression} is governed by \(L\), a natural design objective is
\begin{align*}
\max_{x \in \cX} x^\top \bSigma_{\operatorname{cov},\pb}^{-1} x.
\end{align*}
\citet{pmlr-v291-kim25a} provided a procedure to find a policy \(\pb_G\) satisfying:
\begin{align}\label{equation; p_G}
    \max_{x \in \cX} x^\top \bSigma_{\operatorname{cov},\pb_{G}}^{-1} x \le 4d.
\end{align}
While \(\pb_G\) controls prediction error uniformly over \(\cX\), fixed-confidence BAI requires accurate estimation of \emph{pairwise} target contrasts. Attaining instance-optimal sample complexity therefore calls for an \(\cX\cY\)-style objective adapted to orthogonalized regression. This is nontrivial because \(\bSigma_{\operatorname{cov},\pb}\) depends on \(\pb\) through both the weights and the centering term \(\bar{x}_\pb\).

\paragraph{Challenge: \(\cX\cY\)-design for orthogonalized regression.}
Introducing an active target set \(\cA \subseteq \cZ\), the optimal design objective becomes:
\begin{align*}
\max_{u,u'\in \cA} (u-u')^\top \bSigma_{\operatorname{cov},\pb}^{-1} (u-u').
\end{align*}
This quantity is exactly the relevant \(L\) term in \cref{equation; concentration orthogonalized regression} when the vectors of interest are pairwise differences \(u-u'\). The difficulty is that the optimization problem is non-convex, so even understanding its optimum is nontrivial. Our subsequent algorithmic development gives a constant-factor construction that is sufficient for the final high-probability upper bound.

\section{Lower Bounds}
\label{Section; lower bounds}

We now establish a lower bound on the sample complexity of fixed-confidence transductive BAI in semiparametric bandits. We begin with the second-moment matrix induced by shifted source features.

\begin{definition}
\label{definition; second moment}
For a policy $\pb = (p_1, \dots, p_K) \in \Delta^{(K)}$, we define the second moment of the shifted features $\cX - x_1$ as:
\begin{align}
\bSigma_{-1,\pb} = \sum_{i=1}^K p_i (x_i - x_1)(x_i - x_1)^\top.
\end{align}
Similarly, for any $j \in [K]$, let $\bSigma_{-j,\pb}$ denote the second moment matrix for the features shifted by $x_j$ (i.e., $\cX - x_j$).
\end{definition}

For the corresponding linear bandit problem with source features $\cX-x_1$ and target set $\cZ$, the optimal instance-dependent sample complexity for transductive BAI is \citep{fiez2019sequential}
\begin{align}\label{equation; lower bound transductive linear bandit shifted}
    \bm\tau_{\operatorname{lin}}^\star(\cZ : \cX-x_1) =
 \min_{\pb \in \Delta^{(K)}} \max_{z \in \cZ \setminus \{z^\star\} } \frac{\|z^\star-z\|_{\bSigma_{-1,\pb}^{-1}}^2 }{|(z^\star-z)^\top \theta^\star|^2}.
\end{align}
The next proposition shows that this shifted linear benchmark is unavoidable in the worst case over admissible baseline-shift sequences.

\begin{proposition}[Lower bound for transductive BAI]
\label{proposition; transductive lower bound}
Under Assumption~\ref{assumption; boundedness}, consider any algorithm for fixed-confidence transductive BAI in semiparametric bandits with source features $\cX$ and target features $\cZ$, required to be \(\delta\)-correct for every admissible shift sequence \(\{\nu_t\}\). Then there exists an admissible deterministic shift sequence for which the expected stopping time $\bm\tau(\delta)$ satisfies:
\begin{align*}
\EE[\bm\tau(\delta)] \gtrsim \log\left(\frac{1}{\delta}\right) \times \bm\tau^\star_{\operatorname{lin}}(\cZ : \cX-x_1).
\end{align*}
\end{proposition}

\paragraph{Discussion.}
Proposition~\ref{proposition; transductive lower bound} identifies the fundamental worst-case benchmark for semiparametric BAI over admissible baseline-shift sequences. Two questions remain: whether the bound is attainable, and whether the specific anchor arm $x_1$ matters. The next result addresses the second issue by showing that different anchors change the benchmark by at most a constant factor.

\begin{proposition}[Compatibility of lower bounds]
\label{proposition; transductive compatibility}
For any indices $i,j \in [K]$, the following inequality holds:
\begin{align*}
    \bm\tau^\star_{\operatorname{lin}}(\cZ : \cX-x_i) \le 4 \cdot \bm\tau^\star_{\operatorname{lin}}(\cZ : \cX-x_j).
\end{align*}
\end{proposition}

\paragraph{Discussion.}
Although \eqref{equation; lower bound transductive linear bandit shifted} is written with anchor $x_1$, Proposition~\ref{proposition; transductive compatibility} shows that the benchmark is effectively \textit{anchor-invariant}. 
Conceptually, this means that the hardness of the problem does not depend on an arbitrary reference arm.
Algorithmically, it lets us choose whichever anchor is most convenient for design and analysis.

Since the standard setting \(\cZ=\cX\) is a special case, we immediately obtain the following corollary.
\begin{corollary}[Lower bound: Non-transductive case]
By setting $\cZ=\cX$ in Proposition~\ref{proposition; transductive lower bound}, we recover the corresponding worst-case lower bound for the standard (non-transductive) semiparametric BAI problem:
\begin{align*}
\EE[\bm \tau(\delta)]
&\gtrsim \log\left(\frac{1}{\delta}\right)\times \bm{\tau}_{\operatorname{lin}}^\star(\cX : \cX-x_{1}) \\
&= \log\left(\frac{1}{\delta}\right)\times \bm{\tau}_{\operatorname{lin}}^\star(\cX-x_{1}).
\end{align*}
\end{corollary}

\section{Algorithm}
\label{section; algorithm}

We now present our algorithm for transductive BAI in semiparametric bandits. It follows the standard phase-elimination template from linear bandits \citep{fiez2019sequential}, but replaces the linear-bandit design step with a new \(\cX\cY\)-design tailored to orthogonalized regression.

\paragraph{Roadmap.}
The key technical ingredient is a new \(\cX\cY\)-design for orthogonalized regression, so we present it first. It controls \(\max_{u,u'\in\cA}\|u-u'\|^2_{\bSigma_{\cov,\pb}^{-1}}\), the leading variance term governing pairwise contrasts in orthogonalized regression. Once this contrast-focused design is in place, the rest of the algorithm follows the usual phase-elimination template, with an additional mixture with a semiparametric analogue of G-optimal design used only to control the lower-order source-stability term in \eqref{equation; sample complexity orthogonalized regression}.
\subsection{\texorpdfstring{\(\cX\cY\)-Design for Orthogonalized Regression}{XY-Design for Orthogonalized Regression}}

Fix an active target set \(\cA \subset \RR^d\) (typically \(\cA \subseteq \cZ\)). In analogy with linear-bandit \(\cX\cY\)-design, we define the semiparametric design objective as follows.

\begin{problem}[\(\cX\cY\)-design for orthogonalized regression]
\label{problem; XY-design for orthogonalized regression}
Define the maximum pairwise variance over the set \(\cA\) under policy \(\pb\) as:
\begin{align*}
\mathcal{V}_{\operatorname{cov}}(\cA : \cX, \pb) :=  \max_{u,u' \in \cA} \| u-u' \|_{ \bSigma_{\operatorname{cov}, \pb}^{-1}}^2.
\end{align*}
Our objective is to find a policy that minimizes this quantity:
\begin{align*}
\mathcal{V}_{\operatorname{cov}}^\star(\cA:\cX) := \min_{\pb \in \Delta^{(K)}} \mathcal{V}_{\operatorname{cov}}(\cA : \cX, \pb).
\end{align*}
\end{problem}
This problem is non-convex because \(\bSigma_{\operatorname{cov}, \pb}\) depends on \(\pb\) through both the weights and the centering term \(\bar{x}_\pb\), leading to a cubic dependence on \((p_1, \dots, p_K)\). Fortunately, exact optimization is unnecessary for our purposes. A constant-factor approximation with respect to the right linear-bandit benchmark is enough for our final high-probability upper bound.

\paragraph{Discussion: Sufficiency of Approximation.}
By Proposition~\ref{proposition; transductive lower bound}, the lower-bound benchmark is expressed through the \emph{linear} transductive BAI problem on shifted features \(\cX-x_1\). Therefore, it is enough to build a semiparametric policy whose contrast variances are within a constant factor of the optimal linear \(\cX\cY\)-design value on this shifted instance. This naturally suggests a reduction: compute a linear \(\cX\cY\)-design on \(\cX-x_1\) and convert it into a valid semiparametric policy.

\paragraph{Linear \(\cX\cY\)-design benchmark.}
For the corresponding linear instance on shifted features, define
\begin{align*}
\mathcal{V}_{\operatorname{lin}}^\star(\cA:\cX-x_1) := \min_{\pb \in \Delta^{(K)}} \max_{u,u'\in\cA} \|u-u'\|_{\bSigma_{-1,\pb}^{-1}}^2.
\end{align*}

\paragraph{Proposed Method for \(\cX\cY\)-Design.}
The construction is simple:
\begin{enumerate}
    \item Compute an optimal linear-bandit \(\cX\cY\)-design for the shifted source features \(\cX - x_1\) and active target set \(\cA\). Denote the resulting policy by \(\tilde{\pb} = (\tilde{p}_1, \tilde{p}_2, \dots, \tilde{p}_K)\); necessarily \(\tilde{p}_1=0\) because \(x_1-x_1=0\) carries no information.
    \item Form the semiparametric policy \(\pb_{\operatorname{xor}}\) by mixing \(\tilde{\pb}\) with the anchor arm:
    \[
    \pb_{\operatorname{xor}} = \left(\frac{1}{2}, \frac{1}{2}\tilde{p}_2, \dots, \frac{1}{2}\tilde{p}_K\right).
    \]
\end{enumerate}
Since \(\tilde{p}_1=0\), \(\pb_{\operatorname{xor}}\) is a valid distribution. 
The pseudocode appears in Algorithm~\ref{algorithm; xor}.

\begin{algorithm}[H]
\caption{\texttt{XOR}: \(\cX\cY\)-Design for \textbf{O}rthogonalized \textbf{R}egression}
\label{algorithm; xor}
\begin{algorithmic}[1]
\REQUIRE Target set \(\cA \subset \RR^d\), Source feature set \(\cX\).
\STATE Choose an anchor source arm, denoted by \(x_1\).
\STATE Compute the linear bandit \(\cX\cY\)-design policy \((\tilde{p}_1, \dots, \tilde{p}_K)\) for the active set \(\cA\) using source features \(\cX-x_1\), and set \(\tilde{p}_1=0\).
\STATE \textbf{Return} the policy \(\pb_{\operatorname{xor}}(\cA) = \big(\frac{1}{2}, \frac{\tilde{p}_2}{2}, \dots, \frac{\tilde{p}_K}{2}\big)\).
\end{algorithmic}
\end{algorithm}

\begin{proposition}[Performance of \texttt{XOR} design]
\label{proposition; performance of XY design}
The design \(\pb_{\operatorname{xor}}\) returned by Algorithm~\ref{algorithm; xor} satisfies:
\begin{align*}
\mathcal{V}_{\operatorname{cov}}(\cA:\cX,\pb_{\operatorname{xor}}(\cA)) \le 4 \cdot \mathcal{V}_{\operatorname{lin}}^\star(\cA: \cX-x_1).
\end{align*}
\end{proposition}

\paragraph{Implications of Proposition~\ref{proposition; performance of XY design}.}
Proposition~\ref{proposition; performance of XY design} is the key design result. Although Problem~\ref{problem; XY-design for orthogonalized regression} is non-convex, \texttt{XOR} achieves a constant-factor approximation to the \emph{optimal linear} \(\cX\cY\)-design on the shifted instance. Because our lower bound is stated in terms of exactly that benchmark, this approximation is sufficient for the final upper bound to track the benchmark up to logarithmic factors.

\subsection{Algorithm for Fixed-Confidence Transductive BAI}

We now describe the full elimination algorithm. In phase \(\ell\), with active target set \(\cA_\ell \subseteq \cZ\), we sample \(n_\ell\) source arms according to the mixture policy
\begin{align}\label{equation; p_ell}
\pb_{\ell} = \frac{1}{2} \big(\pb_{\operatorname{xor}}(\cA_\ell) + \pb_{G}\big),
\end{align}
where \(\pb_{\operatorname{xor}}(\cA_\ell)\) comes from Algorithm~\ref{algorithm; xor} and \(\pb_{G}\) is the approximate G-optimal design satisfying \eqref{equation; p_G}. We set the confidence radius to \(\varepsilon_\ell = 2^{-\ell}\) and define
\begin{align*}
\mathcal{V}_{\operatorname{cov}}(\cA_\ell:\cX, \pb_{\ell}) = \max_{u,u' \in \cA_\ell} \| u-u'\|^2_{\bSigma_{\operatorname{cov}, \pb_{\ell}}^{-1}}.
\end{align*}

\paragraph{Rationale for the Mixture Policy.}
The mixture \(\pb_{\ell}\) is designed to control both terms in the orthogonalized-regression concentration bound \eqref{equation; concentration orthogonalized regression}. The \(\cX\cY\)-component \(\pb_{\xor}\) controls contrast variances within the active set and therefore the leading \(1/\Delta^2\) contribution (the \(L\)-term in \eqref{equation; concentration orthogonalized regression}). The G-optimal component \(\pb_G\) controls prediction variance over the source set and stabilizes the lower-order source-stability contribution (the \(M\)-term). Combining them yields contrast efficiency without sacrificing uniform stability.

Based on \(\mathcal{V}_{\operatorname{cov}}(\cA_\ell:\cX, \pb_{\ell})\), we determine the phase length \(n_\ell\) as:
\begin{equation}\label{equation; n_ell}
\begin{aligned}
n_\ell = \bigg \lceil & R_1\frac{\mathcal{V}_{\operatorname{cov}}(\cA_\ell:\cX, \pb_{\ell})}{\varepsilon_\ell^2} \log\!\left(\frac{\mathcal{V}_{\operatorname{cov}}(\cA_\ell:\cX, \pb_{\ell})}{\varepsilon_\ell \delta_\ell}\right)  \\
&+ R_2 \frac{ 32d\sqrt{\mathcal{V}_{\operatorname{cov}}(\cA_\ell:\cX, \pb_{\ell})}}{\varepsilon_\ell} \log\!\left(\frac{d}{\delta_\ell}\right) \bigg \rceil,
\end{aligned}    
\end{equation}
where \(\varepsilon_\ell = 2^{-\ell}\) and \(\delta_\ell= \frac{\delta}{|\cA_\ell|^2\ell(\ell+1)}\).
The constants \(R_1, R_2\) correspond to those in the concentration bound \eqref{equation; sample complexity orthogonalized regression}.
In each phase, the algorithm collects \(n_\ell\) samples under \(\pb_\ell\), fits an orthogonalized-regression estimator \(\hat\theta_\ell\), eliminates targets whose estimated gap from the empirical best exceeds \(\varepsilon_\ell\), and then repeats on the surviving set. The full pseudocode appears in Algorithm~\ref{algorithm; transductive fixed confidence BAI}.

\begin{algorithm}[ht]
\caption{\texttt{SP-BAI}: Fixed-Confidence \textbf{BAI} for \textbf{S}emi\textbf{P}arametric Bandits}
\label{algorithm; transductive fixed confidence BAI}
\begin{algorithmic}[1]
\STATE \textbf{Input:} Source features \(\mathcal{X}=\{x_1,\dots,x_K\}\), Target features \(\cZ = \{z_1, \dots, z_H\}\), Confidence \(\delta \in (0,1)\). Hyperparameters \(R_1, R_2 >0\).
\STATE \textbf{Initialize:} Active set \(\mathcal{A}_1=\cZ\), phase index \(\ell =1\).
\WHILE{\(|\mathcal{A}_\ell|>1\)}
    \STATE Set confidence radius \(\varepsilon_\ell \gets 2^{-\ell}\).
    \STATE Set phase confidence \(\delta_\ell \gets \delta/(|\cA_\ell|^2\ell(\ell+1))\).
    \STATE Compute the sampling policy \(\pb_{\ell}\) for \(\mathcal{A}_\ell\) via \eqref{equation; p_ell}.
    \STATE Calculate phase length \(n_\ell\) via \eqref{equation; n_ell}.
    \STATE \textbf{Data Collection:}
    \STATE \quad Initialize \(\Bb \gets \mathbf{0}_{d\times d}\), \(\bb \gets \mathbf{0}_d\).
    \FOR{\(s=1\) \TO \(n_\ell\)}
        \STATE Sample \(a_s \sim \pb_{\ell}\) and observe reward \(r_s\).
        \STATE Compute centered feature: \(\tilde{x}_{a_s} \gets x_{a_s}-\sum_{i=1}^K \pb_{\ell}(i)x_i\).
        \STATE Update statistics: \(\Bb \gets \Bb + \tilde{x}_{a_s}\tilde{x}_{a_s}^\top\), \(\quad \bb \gets \bb + \tilde{x}_{a_s} r_s\).
    \ENDFOR
    \STATE \textbf{Estimation and Elimination:}
    \STATE Compute estimator: \(\hat{\theta}_{\ell} \gets \big(\Bb + \log(n_\ell/\delta_\ell)\, \Ib_d\big)^{-1} \bb\).
    \STATE Identify empirical best: \(\hat{z}_\ell \in \arg\max_{ z \in \mathcal{A}_\ell} z^\top \hat{\theta}_{\ell}\).
    \STATE Update active set: \(\mathcal{A}_{\ell+1} \gets \{\,  z \in \mathcal{A}_\ell \mid (\hat{z}_\ell-z)^\top \hat{\theta}_{\ell} < \varepsilon_\ell \,\}\).
    \STATE Increment phase: \(\ell \gets \ell+1\).
\ENDWHILE
\STATE \textbf{Output:} The single arm remaining in \(\mathcal{A}_\ell\).
\end{algorithmic}
\end{algorithm}

\paragraph{Computational Efficiency.}
Each phase of Algorithm~\ref{algorithm; transductive fixed confidence BAI} requires two convex optimization subroutines: a linear \(\cX\cY\)-design on the shifted instance (to obtain \(\pb_{\xor}\)) and an approximate G-optimal design (to obtain \(\pb_G\)). Both can be solved efficiently, for example by Frank-Wolfe or multiplicative-weights methods, so the overall procedure is computationally practical.

\section{Upper Bounds and Near-Optimality}
\label{section; upper bound}

We now state the sample-complexity guarantee for Algorithm~\ref{algorithm; transductive fixed confidence BAI}.

\begin{theorem}[High-probability sample complexity bound]
\label{theorem; transductive upper bound}
Under Assumption~\ref{assumption; boundedness}, with probability at least \(1-\delta\), Algorithm~\ref{algorithm; transductive fixed confidence BAI} correctly identifies the best target arm \(z^\star\), and its stopping time satisfies
\begin{align*}
\bm\tau(\delta) \lesssim \bm\tau^\star_{\operatorname{lin}}(\cZ : \cX-x_1) + d^2,
\end{align*}
where the notation \(\lesssim\) suppresses absolute constants and logarithmic factors.
\end{theorem}

\paragraph{Discussion.}
The theorem shows that, whenever the shifted-linear benchmark is finite, our algorithm attains a high-probability sample-complexity upper bound of order \(\bm\tau^\star_{\operatorname{lin}}(\cZ : \cX-x_1)+d^2\) up to logarithmic factors. Combined with Proposition~\ref{proposition; transductive lower bound}, this shows that \(\bm\tau^\star_{\operatorname{lin}}(\cZ : \cX-x_1)\) is unavoidable in general and therefore the right benchmark to compare against.

The additive \(d^2\) term comes from stability requirements for orthogonalized regression. Importantly, it is independent of the instance-specific gaps. On the hard instances that dominate fixed-confidence complexity, the leading term \(\bm\tau^\star_{\operatorname{lin}}(\cZ : \cX-x_1)\) therefore governs the sample complexity.

The standard non-transductive guarantee is recovered immediately by taking \(\cZ=\cX\).
\begin{corollary}[Sample complexity: non-transductive case]
By setting \(\cZ=\cX\) in Theorem~\ref{theorem; transductive upper bound}, we recover the guarantee for the standard (non-transductive) setting:
\begin{align*}
\bm\tau(\delta) \lesssim \bm{\tau}_{\operatorname{lin}}^\star(\cX: \cX-x_{1}) + d^2 = \bm{\tau}_{\operatorname{lin}}^\star(\cX-x_{1}) + d^2,
\end{align*}
where the notation \(\lesssim\) suppresses absolute constants and logarithmic factors.
\end{corollary}
Thus, in both the transductive and standard settings, semiparametric BAI has essentially the same instance-dependent complexity landscape as linear-bandit BAI on the shifted feature space \(\cX - x_1\).

\paragraph{Normalization and General Scales.}
For clarity we present the normalized case \(|\nu_t| \le 1\) with unit sub-Gaussian noise. More generally, if \(|\nu_t| \le R\) and \(\eta_t\) is \(\alpha\)-sub-Gaussian, rescaling rewards by \(1/(R+\alpha)\) reduces to this setting, so the confidence width scales with \(R+\alpha\) and the sample-complexity bound by \((R+\alpha)^2\); see also \citet{pmlr-v291-kim25a}.

\section{Experiments with Synthetic Data}

\subsection{Two Feature Instances}
We consider two synthetic instance families in the standard non-transductive setting \(\cZ=\cX\).
In all synthetic experiments, the reward is generated according to the semiparametric model
\[
r_t = x_{a_t}^\top \theta^\star + \nu_t + \eta_t,
\]
where \(\eta_t \sim \cN(0,1)\), \(\nu_t = 1 + \sin(2t)\) as in \citet{pmlr-v291-kim25a}. All reported numbers are averages over \(100\) independent runs. Unless stated otherwise, we use \(\delta=0.1\) throughout, including the real-data experiments below.

\textbf{Feature set 1: small-gap instance.}\;
We set \(\cX = \{e_1, \dots, e_d, z\}\) with \(z = \cos(\alpha)\, e_1 + \sin(\alpha)\, e_2\).
We also set \(\theta^\star = 2e_1\).
This type of instance is standard in linear-bandit BAI experiments \citep{fiez2019sequential}. We set \(\alpha = 0.2\), which yields a small suboptimality gap \(\Delta_{2} \approx 0.04\).
For this family, we report results for \(d = 10\) with \(\delta = 0.1\).

\textbf{Feature set 2: uniform-feature instance.}\;
For the second family, we sample \(\cX\) with \(|\cX| = K\) uniformly from the unit sphere \(\mathbb{S}^{d-1}\) and again set \(\theta^\star = 2e_1\). We consider \(d=10\), \(K=100\), and \(\delta=0.1\). Detailed standard deviations are reported in the appendix.

\subsection{Baselines and Hyperparameters}
We compare \texttt{SP-BAI} with two semiparametric baselines and one misspecified linear-bandit reference: \texttt{SBE} \citep{pmlr-v291-kim25a}, \texttt{G-Opt}, and \texttt{RAGE} \citep{fiez2019sequential}.

\textbf{SBE.}\;
This is the arm-elimination procedure of \citet{pmlr-v291-kim25a}, which uses G-optimal design for orthogonalized regression.
All other implementation details follow \citet{pmlr-v291-kim25a}.

\textbf{G-Opt.}\;
This oracle one-shot baseline assumes the smallest nonzero gap \(\Delta_2\) is known. We compute an approximate G-optimal design over \(\cX\) using the semiparametric design routine of \citet{pmlr-v291-kim25a}, sample according to that design for the corresponding one-shot budget from their fixed-confidence analysis, then fit orthogonalized regression and output \(\arg\max_x x^\top \hat{\theta}\).

\textbf{RAGE.}\;
RAGE \citep{fiez2019sequential} is a nearly optimal algorithm for linear-bandit BAI, but it does not model the round-wise shared baseline shift. We therefore include it only as a misspecified reference point.

\textbf{Hyperparameters.}\;
We use the same constants \(R_1, R_2\) in \cref{equation; sample complexity orthogonalized regression} throughout and set \(R_1 = R_2 = 1/3\) in every experiment. The confidence level is \(\delta=0.1\) in every experiment. In our implementation of \texttt{SP-BAI}, the anchor arm used inside the \texttt{XOR} design is chosen phase by phase as the current empirical best from the previous phase. Since the theory allows any anchor up to constant factors, this choice is simply a practical heuristic.

\subsection{Simulation Results}
Table~\ref{table; small gap d=10} reports the small-gap result. On this instance, \texttt{SP-BAI} achieves the smallest average stopping time among the semiparametric methods while maintaining zero empirical error over \(100\) runs. By contrast, \texttt{RAGE} stops much earlier at small \(\sigma\) but fails completely, confirming that the linear model is badly misspecified in this setting.

Table~\ref{table; uniform feature K=100} gives a complementary random uniform-feature instance. This example highlights calibration rather than raw stopping time. The oracle one-shot \texttt{G-Opt} baseline is fastest on average, but its error probability rises to \(0.29\). \texttt{SP-BAI} is more conservative, yet it achieves the smallest empirical error (\(0.01\)) among all reported methods.
\begin{table}[t]
\centering
\caption{Small-gap instance, \(d=10\) (100 runs).}
\label{table; small gap d=10}
\small
\begin{tabular}{l S[table-format=7.2] S[table-format=1.4]}
\toprule
& \textbf{Avg. $\bm\tau$} & \textbf{Avg.\ Error Prob.} \\
\midrule
\texttt{SBE} & 1881513.94 & 0.0000 \\
\texttt{G-Opt} & 387106.00 & 0.0000 \\
\textbf{\texttt{SP-BAI}} & \textbf{187924.68} & \textbf{0.0000} \\
\texttt{RAGE} (\(\sigma=1\)) & 22212.00 & 1.0000 \\
\texttt{RAGE} (\(\sigma=3\)) & 199908.00 & 1.0000 \\
\bottomrule
\end{tabular}
\end{table}

\begin{table}[t]
\centering
\caption{Uniform-feature instance, \(d=10\), \(K=100\) (100 runs).}
\label{table; uniform feature K=100}
\small
\begin{tabular}{l S[table-format=7.2] S[table-format=1.4]}
\toprule
& \textbf{Avg.\ $\bm\tau$} & \textbf{Avg.\ Error Prob.} \\
\midrule
\texttt{SBE} & 1523534.64 & 0.1000 \\
\texttt{G-Opt} & 323509.88 & 0.2900 \\
\textbf{\texttt{SP-BAI}} & \textbf{2038886.49} & \textbf{0.0100} \\
\texttt{RAGE} (\(\sigma=1\)) & 4367940.79 & 0.9500 \\
\texttt{RAGE} (\(\sigma=3\)) & 6065302.82 & 0.9400 \\
\bottomrule
\end{tabular}
\end{table}


\section{Real-world Data: Jester Joke Ratings}
\label{section; real_world_jester}

We next test the practical utility of our method on the \emph{Jester} dataset \citep{goldberg2001eigentaste}, which contains continuous user ratings for jokes. Jester exhibits substantial \emph{user-specific baseline variation}: some users systematically rate most jokes high, while others rate most jokes low, regardless of the joke identity.

\paragraph{Experimental Setup.}
We focus on \(K=8\) jokes, namely \(\texttt{ARM\_IDS}=\{7,8,13,15,16,17,18,19\}\). For this joke set, we restrict attention to the maximal complete-case subset, consisting of the \(50{,}699\) users who rated all eight jokes. This lets each pull correspond to sampling a user and then querying any joke without introducing an additional missing-data model. Since the action set is finite, we use one-hot features \(x_i=e_i\). Each pull samples a user uniformly with replacement from these \(50{,}699\) users and observes the corresponding joke rating. The ground-truth best arm, Joke \#19, is defined by the empirical mean over the full subset. As in the synthetic experiments, we use \(R_1=R_2=1/3\) and \(\delta=0.1\).

\subsection{Justification of Semiparametric Reward Model: Impact of Baseline Correction}
\label{subsection; jester_deo}

To isolate the value of baseline correction, we first compare a \textbf{Uniform} sampling strategy with \textbf{DEO} (the approximate G-optimal design method for semiparametric bandits from \citealt{pmlr-v291-kim25a}) in a single-phase setting. Both methods collect a fixed budget of samples and then rank arms by their estimated values, but DEO uses orthogonalized regression to remove the user baseline \(\nu_t\). We report \(100\) runs at budgets \(T \in \{3000,5000\}\).

\textbf{Results.}
Accounting for the user baseline matters already in this simple ranking task. As shown in Table~\ref{tab:jester_toy}, DEO identifies the best joke in \(65\%\) of runs at budget \(3000\), versus \(58\%\) for uniform sampling, and the gap widens to \(69\%\) versus \(58\%\) at budget \(5000\). The naive estimator absorbs the large variance of \(\nu_t\) into the arm means, making it harder to separate the two leading jokes, 19 and 17.

\begin{table}[h]
    \centering
    \renewcommand{\arraystretch}{1.05}
    \caption{Jester toy ranking: probability of ranking Joke \#19 first over 100 runs.}
    \label{tab:jester_toy}
    \begin{tabular}{lcc}
        \toprule
        \textbf{Budget} & \textbf{DEO} & \textbf{Uniform} \\
        \midrule
        \(3000\) & \textbf{0.65} & 0.58 \\
        \(5000\) & \textbf{0.69} & 0.58 \\
        \bottomrule
    \end{tabular}
\end{table}

\subsection{Fixed-confidence BAI on Jester}
\label{subsection; jester_bai}

We next evaluate fixed-confidence BAI on Jester. We compare our \texttt{SP-BAI} algorithm against (i) \texttt{SBE} \citep{pmlr-v291-kim25a}, which uses approximate G-optimal design but not our \(\cX\cY\)-optimization, and (ii) standard MAB baselines following \citet{jamieson2014best}, namely \texttt{LUCB} and \texttt{AE}, run over the noise-proxy grid \(\sigma \in \{1,1.5,2,2.5,3,3.5,4\}\). In the main table, for each MAB baseline we report the \emph{smallest-sample} choice among those whose empirical error probability is below \(10\%\); the full grid is deferred to Appendix~\ref{section; more experiments}.

\textbf{Results.}
Table~\ref{tab:jester_bai} reports averages over 100 runs. \texttt{SP-BAI} achieves a \textbf{97\% success rate} with only \textbf{61,546} average pulls, while \texttt{SBE} attains \textbf{98\%} success but requires \textbf{266,706} pulls, about \(4.3\times\) more samples. Under the ``error \(<0.1\)'' rule, the best standard MAB baselines are \texttt{LUCB} with \(\sigma=3.5\) and \texttt{AE} with \(\sigma=1.5\), and both remain less sample-efficient than \texttt{SP-BAI}.

\begin{table}[t]
    \centering
    \renewcommand{\arraystretch}{1.1}
    \caption{Jester (fixed-confidence BAI): Success probability and sample complexity over 100 runs. For \texttt{LUCB} and \texttt{AE}, we report the smallest-sample choice among \(\sigma\)-values whose empirical error probability is below \(10\%\); the full grid appears in the appendix.}
    \label{tab:jester_bai}
    \vspace{0.25em}
    \begin{tabular}{lcc}
        \toprule
        \textbf{Method} & \textbf{Success Prob.} & \textbf{Avg. Pulls} \\
        \textbf{\texttt{SP-BAI} (Ours)} & \textbf{0.97} & \textbf{\num{61546}} \\
        \texttt{SBE} & 0.98 & \num{266706} \\
        \midrule
        \texttt{LUCB} ($\sigma=3.5$) & 0.96 & \num{136698} \\
        \texttt{AE} ($\sigma=1.5$) & 0.98 & \num{105939} \\
        \bottomrule
    \end{tabular}
\end{table}

\section{Conclusion}
We studied fixed-confidence BAI in semiparametric bandits, where rewards contain an arbitrary baseline shift in addition to a linear treatment effect. We identified the correct shifted-linear complexity benchmark, developed a new \(\cX\cY\)-design for orthogonalized regression, and proved a high-probability sample-complexity upper bound controlled by that benchmark up to logarithmic factors and an additive \(d^2\) term.
Two natural directions for future work are fixed-budget BAI in semiparametric bandits and top-\(m\) identification under the fixed-confidence criterion.

\bibliography{ref}

@inproceedings{tao2018best,
  title={Best arm identification in linear bandits with linear dimension dependency},
  author={Tao, Chao and Blanco, Sa{\'u}l and Zhou, Yuan},
  booktitle={International Conference on Machine Learning},
  pages={4877--4886},
  year={2018},
  organization={PMLR}
}

@inproceedings{abbasi2018best,
  title={Best of both worlds: Stochastic \& adversarial best-arm identification},
  author={Abbasi-Yadkori, Yasin and Bartlett, Peter and Gabillon, Victor and Malek, Alan and Valko, Michal},
  booktitle={Conference on learning theory},
  pages={918--949},
  year={2018},
  organization={PMLR}
}

@article{goldberg2001eigentaste,
  title={Eigentaste: A constant time collaborative filtering algorithm},
  author={Goldberg, Ken and Roeder, Theresa and Gupta, Dhruv and Perkins, Chris},
  journal={information retrieval},
  volume={4},
  number={2},
  pages={133--151},
  year={2001},
  publisher={Springer}
}

@article{kaufmann2016complexity,
  title={On the complexity of best-arm identification in multi-armed bandit models},
  author={Kaufmann, Emilie and Capp{\'e}, Olivier and Garivier, Aur{\'e}lien},
  journal={The Journal of Machine Learning Research},
  volume={17},
  number={1},
  pages={1--42},
  year={2016},
  publisher={JMLR. org}
}

@article{rusmevichientong2010linearly,
  title={Linearly parameterized bandits},
  author={Rusmevichientong, Paat and Tsitsiklis, John N},
  journal={Mathematics of Operations Research},
  volume={35},
  number={2},
  pages={395--411},
  year={2010},
  publisher={INFORMS}
}

@inproceedings{jamieson2014best,
  title={Best-arm identification algorithms for multi-armed bandits in the fixed confidence setting},
  author={Jamieson, Kevin and Nowak, Robert},
  booktitle={2014 48th annual conference on information sciences and systems (CISS)},
  pages={1--6},
  year={2014},
  organization={IEEE}
}

@article{langford2007epoch,
  title={The epoch-greedy algorithm for contextual multi-armed bandits},
  author={Langford, John and Zhang, Tong},
  journal={Advances in neural information processing systems},
  volume={20},
  number={1},
  pages={96--1},
  year={2007},
  publisher={Citeseer}
}

@article{wen2025joint,
  title={Joint Value Estimation and Bidding in Repeated First-Price Auctions},
  author={Wen, Yuxiao and Han, Yanjun and Zhou, Zhengyuan},
  journal={arXiv preprint arXiv:2502.17292},
  year={2025}
}

@article{kennedy2024minimax,
  title={Minimax rates for heterogeneous causal effect estimation},
  author={Kennedy, Edward H and Balakrishnan, Sivaraman and Robins, James M and Wasserman, Larry},
  journal={Annals of statistics},
  volume={52},
  number={2},
  pages={793},
  year={2024}
}

@article{kennedy2023towards,
  title={Towards optimal doubly robust estimation of heterogeneous causal effects},
  author={Kennedy, Edward H},
  journal={Electronic Journal of Statistics},
  volume={17},
  number={2},
  pages={3008--3049},
  year={2023},
  publisher={The Institute of Mathematical Statistics and the Bernoulli Society}
}

@article{greenewald2017action,
  title={Action centered contextual bandits},
  author={Greenewald, Kristjan and Tewari, Ambuj and Murphy, Susan and Klasnja, Predrag},
  journal={Advances in neural information processing systems},
  volume={30},
  year={2017}
}

@article{yang2022minimax,
  title={Minimax optimal fixed-budget best arm identification in linear bandits},
  author={Yang, Junwen and Tan, Vincent},
  journal={Advances in Neural Information Processing Systems},
  volume={35},
  pages={12253--12266},
  year={2022}
}

@article{choi2023semi,
  title={Semi-parametric contextual bandits with graph-Laplacian regularization},
  author={Choi, Young-Geun and Kim, Gi-Soo and Paik, Seunghoon and Paik, Myunghee Cho},
  journal={Information Sciences},
  volume={645},
  pages={119367},
  year={2023},
  publisher={Elsevier}
}

@inproceedings{degenne2020gamification,
  title={Gamification of pure exploration for linear bandits},
  author={Degenne, R{\'e}my and M{\'e}nard, Pierre and Shang, Xuedong and Valko, Michal},
  booktitle={International Conference on Machine Learning},
  pages={2432--2442},
  year={2020},
  organization={PMLR}
}

@inproceedings{xu2018fully,
  title={A fully adaptive algorithm for pure exploration in linear bandits},
  author={Xu, Liyuan and Honda, Junya and Sugiyama, Masashi},
  booktitle={International Conference on Artificial Intelligence and Statistics},
  pages={843--851},
  year={2018},
  organization={PMLR}
}

@article{fiez2019sequential,
  title={Sequential experimental design for transductive linear bandits},
  author={Fiez, Tanner and Jain, Lalit and Jamieson, Kevin G and Ratliff, Lillian},
  journal={Advances in neural information processing systems},
  volume={32},
  year={2019}
}

@article{soare2014best,
  title={Best-arm identification in linear bandits},
  author={Soare, Marta and Lazaric, Alessandro and Munos, R{\'e}mi},
  journal={Advances in Neural Information Processing Systems},
  volume={27},
  year={2014}
}

@article{jedra2020optimal,
  title={Optimal best-arm identification in linear bandits},
  author={Jedra, Yassir and Proutiere, Alexandre},
  journal={Advances in Neural Information Processing Systems},
  volume={33},
  pages={10007--10017},
  year={2020}
}

@inproceedings{krishnamurthy2018semiparametric,
  title={Semiparametric contextual bandits},
  author={Krishnamurthy, Akshay and Wu, Zhiwei Steven and Syrgkanis, Vasilis},
  booktitle={International Conference on Machine Learning},
  pages={2776--2785},
  year={2018},
  organization={PMLR}
}

@inproceedings{kim2019contextual,
  title={Contextual multi-armed bandit algorithm for semiparametric reward model},
  author={Kim, Gi-Soo and Paik, Myunghee Cho},
  booktitle={International Conference on Machine Learning},
  pages={3389--3397},
  year={2019},
  organization={PMLR}
}

@article{abbasi2011improved,
  title={Improved algorithms for linear stochastic bandits},
  author={Abbasi-Yadkori, Yasin and P{\'a}l, D{\'a}vid and Szepesv{\'a}ri, Csaba},
  journal={Advances in neural information processing systems},
  volume={24},
  year={2011}
}

@article{goldenshluger2013linear,
  title={A linear response bandit problem},
  author={Goldenshluger, Alexander and Zeevi, Assaf},
  journal={Stochastic Systems},
  volume={3},
  number={1},
  pages={230--261},
  year={2013},
  publisher={INFORMS}
}

@article{auer2002finite,
  title={Finite-time analysis of the multiarmed bandit problem},
  author={Auer, Peter and Cesa-Bianchi, Nicolo and Fischer, Paul},
  journal={Machine learning},
  volume={47},
  number={2-3},
  pages={235--256},
  year={2002},
  publisher={Springer}
}

@book{LS19bandit-book,
  title = {Bandit Algorithms},
  author = {Lattimore, Tor and Szepesv\'{a}ri, Csaba},
  year = {2019},
  publisher = {Cambridge University Press (preprint)},
}

@InProceedings{pmlr-v291-kim25a,
  title = 	 {Experimental Design for Semiparametric Bandits},
  author =       {Kim, Seok-Jin and Kim, Gi-Soo and Oh, Min-hwan},
  booktitle = 	 {Proceedings of Thirty Eighth Conference on Learning Theory},
  pages = 	 {3215--3252},
  year = 	 {2025},
  editor = 	 {Haghtalab, Nika and Moitra, Ankur},
  volume = 	 {291},
  series = 	 {Proceedings of Machine Learning Research},
  month = 	 {30 Jun--04 Jul},
  publisher =    {PMLR},
  url = 	 {https://proceedings.mlr.press/v291/kim25a.html}
}

@inproceedings{li2019nearly,
  title={Nearly minimax-optimal regret for linearly parameterized bandits},
  author={Li, Yingkai and Wang, Yining and Zhou, Yuan},
  booktitle={Conference on Learning Theory},
  pages={2173--2174},
  year={2019},
  organization={PMLR}
}

@inproceedings{lykouris2018corruptions,
  title={Stochastic Bandits Robust to Adversarial Corruptions},
  author={Lykouris, Thodoris and Mirrokni, Vahab and Paes Leme, Renato},
  booktitle={Proceedings of the 50th Annual ACM SIGACT Symposium on Theory of Computing},
  pages={114--122},
  year={2018},
  publisher={ACM}
}

@InProceedings{pmlr-v139-zhong21a,
  title = {Probabilistic Sequential Shrinking: A Best Arm Identification Algorithm for Stochastic Bandits with Corruptions},
  author = {Zhong, Zixin and Cheung, Wang Chi and Tan, Vincent},
  booktitle = {Proceedings of the 38th International Conference on Machine Learning},
  pages = {12772--12781},
  year = {2021},
  editor = {Meila, Marina and Zhang, Tong},
  volume = {139},
  series = {Proceedings of Machine Learning Research},
  publisher = {PMLR},
  url = {https://proceedings.mlr.press/v139/zhong21a.html}
}





\section*{Checklist}

\begin{enumerate}

\item For all models and algorithms presented, check if you include:
\begin{enumerate}
\item A clear description of the mathematical setting, assumptions, algorithm, and/or model. [Yes]
\item An analysis of the properties and complexity (time, space, sample size) of any algorithm. [Yes]
\item (Optional) Anonymized source code, with specification of all dependencies, including external libraries. [Yes]
\end{enumerate}

\item For any theoretical claim, check if you include:
\begin{enumerate}
\item Statements of the full set of assumptions of all theoretical results. [Yes]
\item Complete proofs of all theoretical results. [Yes]
\item Clear explanations of any assumptions. [Yes]     
\end{enumerate}

\item For all figures and tables that present empirical results, check if you include:
\begin{enumerate}
\item The code, data, and instructions needed to reproduce the main experimental results (either in the supplemental material or as a URL). [Yes]
\item All the training details (e.g., data splits, hyperparameters, how they were chosen). [Yes]
\item A clear definition of the specific measure or statistics and error bars (e.g., with respect to the random seed after running experiments multiple times). [Yes]
\item A description of the computing infrastructure used. (e.g., type of GPUs, internal cluster, or cloud provider). [Yes]
\end{enumerate}

\item If you are using existing assets (e.g., code, data, models) or curating/releasing new assets, check if you include:
\begin{enumerate}
\item Citations of the creator If your work uses existing assets. [Yes]
\item The license information of the assets, if applicable. [Yes]
\item New assets either in the supplemental material or as a URL, if applicable. [Yes]
\item Information about consent from data providers/curators. [Yes]
\item Discussion of sensible content if applicable, e.g., personally identifiable information or offensive content. [Not Applicable]
\end{enumerate}

\item If you used crowdsourcing or conducted research with human subjects, check if you include:
\begin{enumerate}
\item The full text of instructions given to participants and screenshots. [Not Applicable]
\item Descriptions of potential participant risks, with links to Institutional Review Board (IRB) approvals if applicable. [Not Applicable]
\item The estimated hourly wage paid to participants and the total amount spent on participant compensation. [Not Applicable]
\end{enumerate}

\end{enumerate}

\newpage
\appendix
\onecolumn

\aistatstitle{Nearly Optimal Best Arm Identification for Semiparametric Bandits: Supplementary Materials}
\etocdepthtag.toc{mtappendix}
\etocsettagdepth{mtchapter}{none}
\etocsettagdepth{mtappendix}{subsection}
\tableofcontents

\section{Additional Discussion of Related Work}
\label{section; apdx related work}

\paragraph{Semiparametric Bandits}
The semiparametric reward model was introduced for contextual decision making, where the observed reward consists of a structured treatment effect plus an unrestricted baseline shift \citep{greenewald2017action,krishnamurthy2018semiparametric,kim2019contextual}. That literature primarily studies regret minimization or policy learning under changing contexts. The closest precursor to our work is the recent fixed-feature study of \citet{pmlr-v291-kim25a}. They developed the experimental-design viewpoint for semiparametric bandits, established sharp concentration for orthogonalized regression, and proposed a G-optimal-design-based procedure together with an initial BAI guarantee. Our work builds directly on that foundation, but addresses a question left open there: what is the correct instance-dependent benchmark for fixed-confidence BAI, and can it be attained?

\paragraph{Linear-Bandit BAI and Experimental Design}
The broader linear-bandit literature begins with linearly parameterized exploration and regret minimization \citep{rusmevichientong2010linearly,goldenshluger2013linear}. Within that literature, pure exploration and best-arm identification have developed into a distinct direction. Early fixed-confidence guarantees for linear-bandit BAI were established by \citet{soare2014best}, who already highlighted the role of design matrices and confidence ellipsoids. Subsequent works sharpened this picture in several directions: improved adaptivity and computational efficiency \citep{tao2018best,xu2018fully}, asymptotically or instance-optimal fixed-confidence procedures \citep{degenne2020gamification,jedra2020optimal}, fixed-budget minimax formulations \citep{yang2022minimax}, and a broader understanding that, unlike in classical MAB, the geometry of the feature vectors is the main determinant of statistical difficulty.

For our paper, the most relevant thread is transductive and instance-dependent experimental design for linear BAI. \citet{fiez2019sequential} showed that in the transductive setting the correct complexity is governed by the variances of pairwise target contrasts, which leads to \(\cX\cY\)-design rather than G-optimal design. This point is central for understanding our contribution. G-optimal design controls uniform prediction error over the sampling set \(\cX\), whereas BAI depends on the specific directions needed to separate the best candidate from its competitors. In particular, once elimination begins, only a small subset of contrast directions matters, so a design that is optimal for prediction need not be optimal for identification.

Our lower bound is expressed exactly through the shifted transductive linear instance \((\cZ : \cX-x_1)\), and our algorithm can therefore be viewed as importing the linear-bandit design philosophy into the semiparametric model. At the same time, the transfer is not formal. In linear bandits, the relevant covariance matrix is simply the second moment of the chosen design. In semiparametric bandits, orthogonalized regression instead works with centered covariance matrices that themselves depend on the policy through \(\bar x_{\pb}\). Thus, even though the right benchmark remains linear-bandit-like, the design problem and the analysis are genuinely new.

\paragraph{Corrupted and Adversarially Perturbed Bandits}
Another adjacent line of work studies stochastic bandits with adversarial corruptions \citep{lykouris2018corruptions}, corrupted best-arm identification \citep{pmlr-v139-zhong21a}, and best-arm identification in mixed stochastic-adversarial regimes \citep{abbasi2018best}. These models allow perturbations that are effectively arm-dependent and are typically quantified by a corruption budget \(C\) or an adversarial contamination level, with guarantees that deteriorate as the environment becomes less stochastic.

Our semiparametric model is structurally different. The baseline shift \(\nu_t\) is shared across all arms at round \(t\) and is chosen before the learner samples the current action. If one treated this shared shift as an arbitrary corruption, the induced corruption budget could be as large as order \(T\), since a nontrivial perturbation may appear on every round. The problem is tractable here not because the cumulative perturbation is small, but because the perturbation has a special common-component structure. Orthogonalized regression is designed precisely to cancel this shared term in conditional expectation, so the learner can still recover linear-bandit-type instance dependence even when \(\sum_{t=1}^T |\nu_t|\) is large. In this sense, corruption-robust methods are related but target a different benchmark: they provide generic robustness to unstructured contamination, whereas our algorithm exploits a specific semiparametric structure that those methods do not use.

\section{Proof of Proposition~\ref{proposition; transductive lower bound}}\label{section; apdx proof transductive lower bound}

\begin{proof}
The proof is based on an observation-law equivalence between a carefully chosen semiparametric hard instance and a shifted linear bandit instance.

Consider a semiparametric bandit environment defined by the source feature set \(\cX\), the target feature set \(\cZ\), the parameter \(\theta^\star\), and a deterministic shift sequence \(\nu_t \equiv -x_1^\top \theta^\star\) for all \(t \ge 1\).
First, we verify that this environment satisfies Assumption~\ref{assumption; boundedness}. Since \(\|x_1\|_2 \le 1\) and \(\|\theta^\star\|_2 \le 1\), we have \(|\nu_t| \le 1\), which is a valid shift sequence.

In this environment, the reward observed at round \(t\) after playing arm \(a_t \in [K]\) is
\begin{align*}
r_t &= x_{a_t}^\top \theta^\star + \nu_t + \eta_t \\
&= x_{a_t}^\top \theta^\star - x_1^\top \theta^\star + \eta_t \\
&= (x_{a_t} - x_1)^\top \theta^\star + \eta_t.
\end{align*}
Let \(\tilde{x}_i := x_i - x_1\) denote the shifted feature vectors. The resulting observation process is statistically identical to that of a linear bandit with source feature set \(\tilde{\cX} = \{\tilde{x}_1, \dots, \tilde{x}_K\}\) and parameter \(\theta^\star\).

Now consider the transductive identification problem. The goal is to identify
\begin{align*}
z^\star = \arg\max_{z \in \cZ} z^\top \theta^\star.
\end{align*}
Any \(\delta\)-correct algorithm \(\texttt{Alg}\) for the semiparametric bandit must identify \(z^\star\) with probability at least \(1-\delta\) for every admissible shift sequence \(\{\nu_t\}\). Therefore, \(\texttt{Alg}\) must in particular succeed on the hard instance where \(\nu_t \equiv -x_1^\top \theta^\star\).

Because the history distribution \(\{(a_s, r_s)\}_{s \le t}\) generated by the semiparametric model with \(\nu_t = -x_1^\top \theta^\star\) is identical to that of the linear bandit with source features \(\cX - x_1\), any stopping time \(\bm{\tau}\) and recommendation rule valid for the semiparametric problem are also valid for this linear bandit instance.

We invoke the standard information-theoretic lower bound for transductive linear bandit BAI. For a linear bandit with source features \(\tilde{\cX}\) and target features \(\cZ\), the expected sample complexity is lower-bounded by \(\bm{\tau}_{\operatorname{lin}}^\star(\cZ : \tilde{\cX}) \log(1/\delta)\) \citep{fiez2019sequential}, where:
\begin{align*}
\bm{\tau}_{\operatorname{lin}}^\star(\cZ : \tilde{\cX}) = \min_{\pb \in \Delta^{(K)}} \max_{z \in \cZ \setminus \{z^\star\}} \frac{\|z^\star - z\|^2_{(\sum_{i=1}^K p_i \tilde{x}_i \tilde{x}_i^\top)^{-1}}}{((z^\star - z)^\top \theta^\star)^2}.
\end{align*}
Substituting \(\tilde{x}_i = x_i - x_1\), we conclude that for any \(\delta\)-correct semiparametric algorithm:
\begin{align*}
\EE[\bm{\tau}(\delta)] \gtrsim \bm{\tau}_{\operatorname{lin}}^\star(\cZ : \cX - x_1) \log\left(\frac{1}{\delta}\right).
\end{align*}
\end{proof}

\section{Proof of Proposition~\ref{proposition; transductive compatibility}}\label{section; apdx proof transductive compatibility}

\begin{proof}
Without loss of generality, it suffices to prove
\[
\bm{\tau}_{\operatorname{lin}}^\star(\cZ : \cX-x_2)
\le 4\,\bm{\tau}_{\operatorname{lin}}^\star(\cZ : \cX-x_1).
\]
Recall that
\begin{align*}
\bSigma_{-k, \pb} := \sum_{i=1}^K p_i (x_i - x_k)(x_i - x_k)^\top.
\end{align*}
Fix any \(\pb \in \Delta^{(K)}\). Since
\[
x_i-x_1=(x_i-x_2)+(x_2-x_1),
\]
the matrix inequality \((a+b)(a+b)^\top \preceq 2aa^\top+2bb^\top\) gives
\begin{align*}
\bSigma_{-1,\pb}
&= \sum_{i=1}^K p_i (x_i - x_1)(x_i - x_1)^\top \\
&\preceq 2 \sum_{i=1}^K p_i (x_i - x_2)(x_i - x_2)^\top + 2 (x_1-x_2)(x_1-x_2)^\top \\
&= 2\bSigma_{-2,\pb} + 2 (x_1-x_2)(x_1-x_2)^\top .
\end{align*}

Now define a nonnegative weight vector \(\qb\) by
\begin{align*}
q_1 := 2+2p_1, \qquad q_i := 2p_i \;\; \text{for } i \ge 2.
\end{align*}
Then \(\sum_{i=1}^K q_i = 4\), and
\begin{align*}
\sum_{i=1}^K q_i (x_i-x_2)(x_i-x_2)^\top
= 2\bSigma_{-2,\pb} + 2 (x_1-x_2)(x_1-x_2)^\top.
\end{align*}
Hence, for every vector \(y\),
\begin{align*}
\|y\|^2_{\bSigma_{-1,\pb}^{-1}}
&\ge \left\|y\right\|^2_{\left(\sum_{i=1}^K q_i (x_i-x_2)(x_i-x_2)^\top\right)^{-1}}.
\end{align*}
Let \(\bar{\qb}:=\qb/4 \in \Delta^{(K)}\). By homogeneity of inverse norms,
\begin{align*}
\left\|y\right\|^2_{\left(\sum_{i=1}^K q_i (x_i-x_2)(x_i-x_2)^\top\right)^{-1}}
= \frac{1}{4}
\left\|y\right\|^2_{\left(\sum_{i=1}^K \bar q_i (x_i-x_2)(x_i-x_2)^\top\right)^{-1}} .
\end{align*}

Therefore, for this fixed \(\pb\),
\begin{align*}
\max_{z \in \cZ\setminus\{z^\star\}}
\frac{\|z^\star-z\|^2_{\bSigma_{-1,\pb}^{-1}}}{|(z^\star-z)^\top\theta^\star|^2}
&\ge \frac{1}{4}
\max_{z \in \cZ\setminus\{z^\star\}}
\frac{\|z^\star-z\|^2_{\left(\sum_{i=1}^K \bar q_i (x_i-x_2)(x_i-x_2)^\top\right)^{-1}}}{|(z^\star-z)^\top\theta^\star|^2} \\
&\ge \frac{1}{4}\bm{\tau}_{\operatorname{lin}}^\star(\cZ : \cX-x_2),
\end{align*}
where the second step uses only that \(\bar{\qb}\in\Delta^{(K)}\), so the value at \(\bar{\qb}\) is at least the minimum over all designs. Since this lower bound holds for every \(\pb\in\Delta^{(K)}\), taking the minimum over \(\pb\) yields
\begin{align*}
\bm{\tau}_{\operatorname{lin}}^\star(\cZ : \cX-x_1)
\ge \frac{1}{4}\bm{\tau}_{\operatorname{lin}}^\star(\cZ : \cX-x_2).
\end{align*}
Rearranging yields
\[
\bm{\tau}_{\operatorname{lin}}^\star(\cZ : \cX-x_2)
\le 4\,\bm{\tau}_{\operatorname{lin}}^\star(\cZ : \cX-x_1).
\]
\end{proof}

\section{Proof of Proposition~\ref{proposition; performance of XY design}}\label{section; apdx proof XY allocation}
\begin{proof}
Let \((\tilde{p}_1, \dots, \tilde p_K)\) be an optimal \(\cX\cY\)-design for the shifted linear instance with source features \(\cX-x_1\) and target set \(\cA\). By construction, \(\tilde p_1 =0\). Define
\[
\widetilde{\bSigma}
:=\sum_{i=1}^{K} \tilde p_i(x_i-x_1)(x_i-x_1)^\top
=\sum_{i=2}^{K} \tilde p_i(x_i-x_1)(x_i-x_1)^\top .
\]
Applying Lemma 9 from \citet{pmlr-v291-kim25a} to the mixture policy \(\pb_{\operatorname{xor}}(\cA)\), we obtain
\begin{align*}
\bSigma_{\operatorname{cov}, \pb_{\operatorname{xor}}(\cA)} \succeq \frac{1}{4}\widetilde{\bSigma}.
\end{align*}
Since inverse quadratic forms reverse the PSD order, for every contrast \(y\in\RR^d\),
\begin{align*}
\|y\|^2_{\bSigma_{\operatorname{cov}, \pb_{\operatorname{xor}}(\cA)}^{-1}}
\le 4\|y\|^2_{\widetilde{\bSigma}^{-1}}.
\end{align*}
Now take \(y=x-x'\) for arbitrary \(x,x' \in \cA\). Then
\begin{align*}
\| x-x'\|^2_{ \bSigma_{\operatorname{cov}, \pb_{\operatorname{xor}}(\cA)}^{-1}}
&\le 4\| x-x'\|^2_{\widetilde{\bSigma}^{-1}} \\
&= 4\| (x-x_1)-(x'-x_1)\|^2_{\widetilde{\bSigma}^{-1}}\\
&\le 4\mathcal{V}_{\operatorname{lin}}^\star(\cA: \cX-x_1),
\end{align*}
where the last step is exactly the defining optimality property of \(\tilde{\pb}\) for the shifted linear instance. Taking the maximum over \(x,x'\in\cA\), we conclude that
\begin{align*}
\mathcal{V}_{\operatorname{cov}}(\cA:\cX, \pb_{\operatorname{xor}}(\cA))
\le 4\mathcal{V}^\star_{\operatorname{lin}}(\cA: \cX-x_1).
\end{align*}
\end{proof}

\section{Proof of Theorem~\ref{theorem; transductive upper bound}}\label{section; apdx proof transductive upper bound}

\subsection{Good Events and Correctness}
For each target arm \(z \in \cZ \setminus \{z^\star\}\), define the gap
\[
\Delta_z := (z^\star-z)^\top \theta^\star,
\qquad
\Delta_{\min}:=\min_{z \in \cZ \setminus \{z^\star\}}\Delta_z.
\]
We also define the deterministic phase cutoff
\[
L_\star := \max\!\left\{1,\left\lceil \log_2\!\left(\frac{4}{\Delta_{\min}}\right) \right\rceil\right\}.
\]

\begin{definition}[Good events]
For each phase \(\ell \in [L_\star]\), let \(\Ecr_\ell\) be the event that
\begin{align*}
\big|(z-z')^\top (\hat{\theta}_\ell - \theta^\star)\big| \le \varepsilon_\ell
\qquad
\text{for all } z,z' \in \cA_\ell.
\end{align*}
Define \(\Ecr_\star := \bigcap_{\ell=1}^{L_\star} \Ecr_\ell\).
\end{definition}

\begin{lemma}
\label{lemma; transductive good event}
For every phase \(\ell \in [L_\star]\), \(\PP[\Ecr_\ell] \ge 1-\delta/(\ell(\ell+1))\). Consequently,
\[
\PP[\Ecr_\star] \ge 1-\delta.
\]
\end{lemma}
\begin{proof}
Fix a phase \(\ell\), and condition on the history at the start of that phase. Under this conditioning, the active set \(\cA_\ell\), the phase confidence level \(\delta_\ell = \delta/(|\cA_\ell|^2\ell(\ell+1))\), the sampling distribution \(\pb_\ell\), and the phase length \(n_\ell\) are all deterministic. We prove a conditional probability bound that is uniform over the conditioning event, and then remove the conditioning at the end.

Write \(\pb_\ell = \frac12(\pb_{\operatorname{xor}}(\cA_\ell)+\pb_G)\), and let
\[
\bar x_{\operatorname{xor}} := \sum_i \pb_{\operatorname{xor}}(\cA_\ell)(i)x_i,
\qquad
\bar x_G := \sum_i \pb_G(i)x_i.
\]
Viewing \(\pb_\ell\) as the mixture that first chooses between \(\pb_{\operatorname{xor}}(\cA_\ell)\) and \(\pb_G\) with probability \(1/2\) each, the law of total covariance gives
\begin{align*}
\bSigma_{\operatorname{cov},\pb_\ell}
= \frac12 \bSigma_{\operatorname{cov},\pb_{\operatorname{xor}}(\cA_\ell)}
+ \frac12 \bSigma_{\operatorname{cov},\pb_G}
+ \frac14 (\bar x_{\operatorname{xor}}-\bar x_G)(\bar x_{\operatorname{xor}}-\bar x_G)^\top,
\end{align*}
where the final PSD term is the covariance of the two component means. In particular, the mixture covariance dominates each component covariance up to a factor of \(1/2\), so
\begin{align}
\bSigma_{\operatorname{cov},\pb_\ell}
\succeq \frac12 \bSigma_{\operatorname{cov},\pb_{\operatorname{xor}}(\cA_\ell)},
\qquad
\bSigma_{\operatorname{cov},\pb_\ell}
\succeq \frac12 \bSigma_{\operatorname{cov},\pb_G}.
\label{equation; covariance domination mixture}
\end{align}

Hence, for every contrast \(y=z-z'\) with \(z,z'\in\cA_\ell\),
\begin{align*}
\|y\|^2_{\bSigma_{\operatorname{cov},\pb_\ell}^{-1}}
&\le 2\|y\|^2_{\bSigma_{\operatorname{cov},\pb_{\operatorname{xor}}(\cA_\ell)}^{-1}} \\
&\le 2\,\mathcal{V}_{\operatorname{cov}}(\cA_\ell:\cX,\pb_{\operatorname{xor}}(\cA_\ell)) \\
&\le 8\,\mathcal{V}_{\operatorname{lin}}^\star(\cA_\ell:\cX-x_1),
\end{align*}
where the last step uses Proposition~\ref{proposition; performance of XY design}. Thus, in the concentration bound \eqref{equation; concentration orthogonalized regression}, the leading variance parameter \(L\) for the contrast set \(\cA_\ell-\cA_\ell\) is controlled, up to an absolute constant, by \(\mathcal{V}_{\operatorname{cov}}(\cA_\ell:\cX,\pb_\ell)\).

We next control the source-only quantity \(M\) in \eqref{equation; concentration orthogonalized regression}. By \eqref{equation; covariance domination mixture} and the definition of \(\pb_G\),
\begin{align*}
\|x_i\|_{\bSigma_{\operatorname{cov},\pb_\ell}^{-1}}^2
\le 2 \|x_i\|_{\bSigma_{\operatorname{cov},\pb_G}^{-1}}^2
\le 8d
\qquad \text{for every } i \in [K].
\end{align*}
Moreover,
\begin{align*}
\|x_i-\bar x_{\pb_\ell}\|_{\bSigma_{\operatorname{cov},\pb_\ell}^{-1}}
&\le \sqrt{2}\,\|x_i-\bar x_{\pb_\ell}\|_{\bSigma_{\operatorname{cov},\pb_G}^{-1}} \\
&= \sqrt{2}\,\left\|\sum_{j=1}^K \pb_\ell(j)(x_i-x_j)\right\|_{\bSigma_{\operatorname{cov},\pb_G}^{-1}} \\
&\le \sqrt{2}\sum_{j=1}^K \pb_\ell(j)\big(\|x_i\|_{\bSigma_{\operatorname{cov},\pb_G}^{-1}}+\|x_j\|_{\bSigma_{\operatorname{cov},\pb_G}^{-1}}\big) \\
&\le \sqrt{2}\cdot 2 \sqrt{4d}
= 4\sqrt{2d},
\end{align*}
and therefore
\begin{align}
\max_{i\in[K]}\|x_i-\bar x_{\pb_\ell}\|_{\bSigma_{\operatorname{cov},\pb_\ell}^{-1}}^2 \le 32d.
\label{equation; lower-order control transductive}
\end{align}

Applying the concentration bound \eqref{equation; concentration orthogonalized regression} conditionally on the phase-\(\ell\) history to each vector in the finite set \(\cA_\ell-\cA_\ell\), using the \(L\)-control above together with the source-only \(M\)-bound \eqref{equation; lower-order control transductive}, and taking a union bound over at most \(|\cA_\ell|^2\) contrasts, we obtain
\[
\PP[\Ecr_\ell] \ge 1-|\cA_\ell|^2\delta_\ell
= 1-\frac{\delta}{\ell(\ell+1)}
\]
conditionally on the phase-\(\ell\) history, provided the absolute constants \(R_1,R_2\) in \eqref{equation; n_ell} are chosen large enough. Since the bound is uniform over the conditioning event, it also holds unconditionally. Finally,
\[
\sum_{\ell=1}^{\infty}\frac{\delta}{\ell(\ell+1)}
\le \delta,
\]
so \(\PP[\Ecr_\star]\ge 1-\delta\).
\end{proof}

\begin{lemma}[Correctness]
\label{lemma; transductive correctness}
Under \(\Ecr_\star\), the best arm \(z^\star\) is never eliminated, and after phase \(L_\star\) the active set is exactly \(\{z^\star\}\).
\end{lemma}
\begin{proof}
We first show that \(z^\star\) is never eliminated. Suppose \(z^\star \in \cA_\ell\). If \(z^\star\) were removed in phase \(\ell\), then by the elimination rule there would exist some \(\hat z_\ell \in \cA_\ell\) such that
\[
(\hat z_\ell-z^\star)^\top \hat\theta_\ell \ge \varepsilon_\ell.
\]
Since \(z^\star\) is optimal, \((\hat z_\ell-z^\star)^\top\theta^\star < 0\), hence
\[
(\hat z_\ell-z^\star)^\top(\hat\theta_\ell-\theta^\star)\ge \varepsilon_\ell,
\]
which contradicts \(\Ecr_\ell\). Therefore \(z^\star\) survives every phase.

Next consider any suboptimal arm \(z \in \cA_\ell\) with \(\Delta_z > 2\varepsilon_\ell\). Since \(\hat z_\ell\) maximizes \(u^\top \hat\theta_\ell\) over \(u \in \cA_\ell\),
\begin{align*}
(\hat z_\ell-z)^\top \hat\theta_\ell
\ge (z^\star-z)^\top \hat\theta_\ell
\ge \Delta_z - \varepsilon_\ell
> \varepsilon_\ell,
\end{align*}
where the second step uses \(\Ecr_\ell\). Thus every arm with gap larger than \(2\varepsilon_\ell\) is removed in phase \(\ell\), so
\[
\cA_{\ell+1}\subseteq \{z \in \cZ : \Delta_z \le 2\varepsilon_\ell\}.
\]
For the base case \(\ell=1\), we have \(\cA_1=\cZ\), and for every \(z \in \cZ\),
\[
\Delta_z
= (z^\star-z)^\top \theta^\star
\le \|z^\star-z\|_2 \|\theta^\star\|_2
\le 2
= 4\varepsilon_1.
\]
Hence
\[
\cA_1 \subseteq \{z \in \cZ : \Delta_z \le 4\varepsilon_1\} \cup \{z^\star\}.
\]
Using \(\cA_{\ell+1}\subseteq \{z \in \cZ : \Delta_z \le 2\varepsilon_\ell\}\) and \(\varepsilon_{\ell+1}=\varepsilon_\ell/2\), an immediate induction gives, for every \(\ell \ge 1\),
\[
\cA_\ell \subseteq \{z \in \cZ : \Delta_z \le 4\varepsilon_\ell\} \cup \{z^\star\}.
\]
Finally, by the definition of \(L_\star\), we have \(4\varepsilon_{L_\star} \le \Delta_{\min}\). Therefore no suboptimal arm can remain after phase \(L_\star\), and the active set is \(\{z^\star\}\).
\end{proof}

\subsection{Bounding Sample Complexity}

Let
\[
\tau_\star := \bm\tau_{\operatorname{lin}}^\star(\cZ : \cX-x_1).
\]
Fix an optimal design \(\pb^\star \in \Delta^{(K)}\) for \(\tau_\star\), and write
\[
\bSigma^\star := \sum_{i=1}^K p_i^\star (x_i-x_1)(x_i-x_1)^\top .
\]
Under \(\Ecr_\star\), Lemma~\ref{lemma; transductive correctness} shows that every arm \(z \in \cA_\ell\setminus\{z^\star\}\) satisfies
\[
\Delta_z \le 4\varepsilon_\ell.
\]

We first compare the linear design value for \(\cA_\ell\) with the full-instance lower-bound benchmark. For any policy \(\pb\),
\[
\max_{u,u' \in \cA_\ell}\|u-u'\|^2_{\bSigma_{-1,\pb}^{-1}}
\le 4\max_{z \in \cA_\ell\setminus\{z^\star\}}\|z^\star-z\|^2_{\bSigma_{-1,\pb}^{-1}},
\]
because every pairwise difference can be written as
\[
u-u' = (u-z^\star) - (u'-z^\star)
\]
and \(\|a-b\|^2_{\Ab^{-1}} \le 2\|a\|^2_{\Ab^{-1}} + 2\|b\|^2_{\Ab^{-1}}\).
Evaluating this bound at \(\pb^\star\), we obtain
\begin{align}
\mathcal{V}_{\operatorname{lin}}^\star(\cA_\ell:\cX-x_1)
&\le 4\max_{z \in \cA_\ell\setminus\{z^\star\}}
\|z^\star-z\|^2_{(\bSigma^\star)^{-1}} \notag\\
&\le 4\tau_\star \max_{z \in \cA_\ell\setminus\{z^\star\}} \Delta_z^2 \\
&\le 64 \tau_\star \varepsilon_\ell^2.
\label{equation; active set linear variance bound}
\end{align}
Equation~\eqref{equation; active set linear variance bound} is the key step that connects elimination to the benchmark \(\tau_\star\). The first inequality converts arbitrary pairwise contrasts inside \(\cA_\ell\) into contrasts against the true best arm \(z^\star\). The second inequality uses the defining property of the optimal benchmark design \(\pb^\star\): for every competitor \(z \neq z^\star\),
\[
\|z^\star-z\|^2_{(\bSigma^\star)^{-1}} \le \tau_\star \Delta_z^2.
\]
The last step then uses the consequence of the good event \(\Ecr_\star\), namely that every suboptimal arm surviving into phase \(\ell\) must satisfy \(\Delta_z \le 4\varepsilon_\ell\). Thus, once the active set has shrunk to near-optimal arms, its entire linear \(\cX\cY\)-design value is automatically of order \(\tau_\star \varepsilon_\ell^2\). This is exactly what allows the phase length to track the instance-dependent benchmark rather than a cruder worst-case quantity.

By \eqref{equation; covariance domination mixture} and Proposition~\ref{proposition; performance of XY design},
\begin{align}
\mathcal{V}_{\operatorname{cov}}(\cA_\ell:\cX,\pb_\ell)
&\le 2\,\mathcal{V}_{\operatorname{cov}}(\cA_\ell:\cX,\pb_{\operatorname{xor}}(\cA_\ell)) \notag\\
&\le 8\,\mathcal{V}_{\operatorname{lin}}^\star(\cA_\ell:\cX-x_1)
\le 512 \tau_\star \varepsilon_\ell^2.
\label{equation; active set semiparametric variance bound}
\end{align}

Substituting \eqref{equation; active set semiparametric variance bound} into \eqref{equation; n_ell}, we find
\begin{align*}
n_\ell
&\lesssim \tau_\star \log\!\left(\frac{\tau_\star \varepsilon_\ell}{\delta_\ell}\right)
 + d\sqrt{\tau_\star}\log\!\left(\frac{d}{\delta_\ell}\right).
\end{align*}
Since \(\delta_\ell = \delta/(|\cA_\ell|^2\ell(\ell+1))\), \(|\cA_\ell| \le H\), \(\varepsilon_\ell = 2^{-\ell}\), and \(L_\star = \cO(\max\{1,\log(1/\Delta_{\min})\})\), all logarithmic factors above are polylogarithmic in \(d\), \(H\), \(1/\delta\), and \(1/\Delta_{\min}\). Therefore,
\[
n_\ell \le \tilde{\cO}(\tau_\star + d\sqrt{\tau_\star}).
\]
Summing over \(\ell = 1,\dots,L_\star\) yields
\begin{align*}
\sum_{\ell=1}^{L_\star} n_\ell
&\le \tilde{\cO}\!\left(L_\star(\tau_\star + d\sqrt{\tau_\star})\right) \\
&\le \tilde{\cO}\!\left(L_\star(\tau_\star + d^2)\right),
\end{align*}
where the second step uses \(d\sqrt{\tau_\star} \le \tfrac12(d^2+\tau_\star)\).
Finally, on \(\Ecr_\star\) the algorithm stops by phase \(L_\star\) by Lemma~\ref{lemma; transductive correctness}, so
\[
\bm\tau(\delta) \le \sum_{\ell=1}^{L_\star} n_\ell \le \tilde{\cO}(\tau_\star + d^2).
\]
This proves Theorem~\ref{theorem; transductive upper bound}.

\section{More on Experiments}
\label{section; more experiments}

\subsection{Detailed Synthetic Results}
We summarize the synthetic experiments again, now including empirical standard deviations of the stopping time. All tables below are based on \(100\) independent runs with \(R_1=R_2=1/3\). Tables~\ref{tab:sim-summary-smallgapd10} and \ref{tab:sim-summary-uniformk100} correspond to the synthetic instances in the main text, and include the full \texttt{RAGE} grid used to diagnose linear-model misspecification.
\begin{table}[H]
\centering
\caption{Simulation Results Summary: Small-gap, \(d=10\)}
\label{tab:sim-summary-smallgapd10}
\begin{tabular}{
    l
    S[table-format=7.2]  
    S[table-format=7.2]  
    S[table-format=1.4]  
}
\toprule
{Algorithm} & {Avg.\ $\bm \tau$} & {Std.\ $\bm \tau$} & {Avg.\ Error Prob.} \\
\midrule
\texttt{SBE} & 1881513.94 & 1228628.36 & 0.0000 \\
\texttt{G-Opt} & 387106.00 & 0.00 & 0.0000 \\
\texttt{SP-BAI} & 187924.68 & 112295.59 & 0.0000 \\
\texttt{RAGE} (\(\sigma=1\)) & 22212.00 & 0.00 & 1.0000 \\
\texttt{RAGE} (\(\sigma=2\)) & 88848.00 & 0.00 & 1.0000 \\
\texttt{RAGE} (\(\sigma=3\)) & 199908.00 & 0.00 & 1.0000 \\
\texttt{RAGE} (\(\sigma=4\)) & 355392.00 & 0.00 & 1.0000 \\
\bottomrule
\end{tabular}
\end{table}

\begin{table}[H]
\centering
\caption{Simulation Results Summary: Uniform-feature, \(d=10\), \(K=100\)}
\label{tab:sim-summary-uniformk100}
\begin{tabular}{
    l
    S[table-format=7.2]  
    S[table-format=7.2]  
    S[table-format=1.4]  
}
\toprule
{Algorithm} & {Avg.\ $\bm \tau$} & {Std.\ $\bm \tau$} & {Avg.\ Error Prob.} \\
\midrule
\texttt{SBE} & 1523534.64 & 2822607.36 & 0.1000 \\
\texttt{G-Opt} & 323509.88 & 904131.81 & 0.2900 \\
\texttt{SP-BAI} & 2038886.49 & 2930198.90 & 0.0100 \\
\texttt{RAGE} (\(\sigma=1\)) & 4367940.79 & 1924666.73 & 0.9500 \\
\texttt{RAGE} (\(\sigma=2\)) & 5415498.52 & 1952295.28 & 0.9500 \\
\texttt{RAGE} (\(\sigma=3\)) & 6065302.82 & 2243377.14 & 0.9400 \\
\texttt{RAGE} (\(\sigma=4\)) & 6428781.60 & 2434510.77 & 0.9600 \\
\bottomrule
\end{tabular}
\end{table}

The integrated \texttt{RAGE} rows make the misspecification pattern easy to read across instances. On the reported small-gap instance, \texttt{RAGE} fails uniformly across all tested noise scales \(\sigma \in \{1,2,3,4\}\): the empirical error probability is \(1.0\), while larger \(\sigma\) only increases the stopping time. The uniform-feature instance shows the same qualitative behavior, though less starkly, with error probabilities remaining around \(0.94\) to \(0.96\).

\subsection{Additional Jester Results}
Table~\ref{tab:jester_appendix_full_grid} reports the full \(\sigma\)-grid used for the standard MAB baselines in the Jester fixed-confidence experiment. In the main text, we report only the smallest-sample configuration whose empirical error rate is below \(10\%\); the full grid is included here for transparency.

\begin{table}[H]
\centering
\caption{Jester fixed-confidence experiment: full \(\sigma\)-grid over 100 runs at \(\delta=0.1\).}
\label{tab:jester_appendix_full_grid}
\small
\begin{tabular}{l S[table-format=1.3] S[table-format=6.1]}
\toprule
{Method} & {Success Prob.} & {Avg. Pulls} \\
\midrule
\texttt{SP-BAI} & 0.970 & 61546.4 \\
\texttt{SBE} & 0.980 & 266706.2 \\
\midrule
\texttt{LUCB} ($\sigma=1$) & 0.300 & 64.0 \\
\texttt{LUCB} ($\sigma=1.5$) & 0.460 & 3742.8 \\
\texttt{LUCB} ($\sigma=2$) & 0.540 & 12677.2 \\
\texttt{LUCB} ($\sigma=2.5$) & 0.740 & 31959.7 \\
\texttt{LUCB} ($\sigma=3$) & 0.900 & 82030.7 \\
\texttt{LUCB} ($\sigma=3.5$) & 0.960 & 136698.2 \\
\texttt{LUCB} ($\sigma=4$) & 0.960 & 181116.9 \\
\midrule
\texttt{AE} ($\sigma=1$) & 0.660 & 14286.9 \\
\texttt{AE} ($\sigma=1.5$) & 0.980 & 105938.8 \\
\texttt{AE} ($\sigma=2$) & 1.000 & 185485.2 \\
\texttt{AE} ($\sigma=2.5$) & 1.000 & 340073.2 \\
\texttt{AE} ($\sigma=3$) & 1.000 & 495855.5 \\
\texttt{AE} ($\sigma=3.5$) & 1.000 & 627211.7 \\
\texttt{AE} ($\sigma=4$) & 1.000 & 870003.1 \\
\bottomrule
\end{tabular}
\end{table}

\end{document}